\documentclass{article}
\usepackage{graphicx} % Required for inserting images
\usepackage[final]{nips_2017}  % DO NOT CHANGE THIS

\usepackage{amsmath}
\usepackage{amsthm}
\usepackage{amsfonts}
\usepackage{epsfig}
\usepackage{rotating}
\usepackage{algpseudocode}
\usepackage{graphicx}
\usepackage{comment}
\usepackage{multirow}
\usepackage{xcolor}
\usepackage{algorithm}
\usepackage{url}
\usepackage{epsfig}
\usepackage{rotating}
\usepackage{cases}
\usepackage{mathtools}
\usepackage{dsfont}
\usepackage{hyperref}
\usepackage{cleveref}
\usepackage{longtable}
\usepackage{lipsum}
\usepackage{subcaption}

\usepackage{makecell}
\usepackage{amsmath,amsfonts,amsthm} % Math packages which might be useful for equations
\usepackage{tikz} % For tikz figures (to draw arrow diagrams, see a guide how to use them)
\usepackage{tikz-cd}
\usetikzlibrary{positioning,arrows} % Adding libraries for arrows
\usetikzlibrary{decorations.pathreplacing} % Adding libraries for decorations and paths
\usepackage{tikzsymbols} % For amazing symbols ;) https://mirror.hmc.edu/ctan/graphics/pgf/contrib/tikzsymbols/tikzsymbols.pdf 
\usepackage{blindtext} % To add some blind text in your paper
\usepackage{tikz}
\usetikzlibrary{fit,tikzmark}

\usepackage{subcaption}

\newtheorem{lemma}{Lemma}

\newtheorem{theorem}{Theorem}

\newtheorem{proposition}{Proposition}

\newtheorem{assumption}{Assumption}

\newcommand{\rp}[1]{\left( #1 \right)}
\newcommand{\A}{\mathcal{A}}

\renewcommand{\S}{\mathcal{S}}

\newcommand{\expvalDist}[2]{\mathbb{E}_{#1} \left[ #2 \right]}

\newcommand{\icAlg}{DCC-QL }

\title{Multi-Agent Reinforcement Learning for Task Offloading in Wireless Edge Networks}
\date{}

\author{%
  Andrea Fox \\
  LIA, Avignon University\\
  Avignon, France \\
  \texttt{andrea.fox@univ-avignon.fr} \\
  \And
  Francesco De Pellegrini \\
  LIA, Avignon University\\
  Avignon, France \\
  \texttt{francesco.de-pellegrini@univ-avignon.fr} \\
  \And
  Eitan Altman \\
  INRIA, Sophia Antipolis, France \\
  \texttt{eitan.altman@inria.fr} \\
}

\begin{document}

\maketitle

\begin{abstract}
    In edge computing systems, autonomous agents must make fast local decisions while competing for shared resources\footnote{The code used in our experiments can be found at \url{https://github.com/Andrea-Fox/multiAgentTaskOffloading}.}\footnote{This work has been partially supported by the French National Research Agency (ANR) within the PARFAIT project (ANR-21-CE25-0013).}. Existing MARL methods often resume to centralized critics or frequent communication, which fail under limited observability and communication constraints. We propose a decentralized framework in which each agent solves a constrained Markov decision process (CMDP), coordinating implicitly through a shared constraint vector. For the specific case of offloading, e.g., constraints prevent overloading shared server resources. Coordination constraints are updated infrequently and act as a lightweight coordination mechanism. They enable agents to align with global resource usage objectives but require little direct communication. Using safe reinforcement learning, agents learn policies that meet both local and global goals. We establish theoretical guarantees under mild assumptions and validate our approach experimentally, showing improved performance over centralized and independent baselines, especially in large-scale settings.
\end{abstract}
\section{Introduction}
In decentralized systems, local decisions taken by agents collectively influence the outcome of other agents' actions or the availability of shared resources, with direct impact on individual performance. In mobile edge computing (MEC), for instance, this is the case of edge offloading techniques. They permit multiple devices to independently decide whether to offload computations to a shared edge server or to process them locally. But, while each agent operates based on local observations and limitations, such as battery level, workload, or latency requirements, the collective offloading decisions directly affect server responsiveness and network congestion. Actually, when too many agents offload simultaneously, the resulting overload can degrade performance across the system.
In such settings, each agent’s objective depends not only on its own decisions but also on the aggregated behavior of others, introducing a coordination challenge. This challenge is further amplified in environments with communication delays or asynchronous agent behavior, where real-time coordination is difficult or infeasible. Designing scalable learning methods that support implicit coordination under these system technical constraints is essential for efficient and robust distributed operation.

% Most existing multi-agent reinforcement learning (MARL) algorithms are ill-suited for this scenario. Approaches such as QMIX, VDN, MADDPG, MAPPO, and COMA rely on centralized critics, joint training, or synchronized updates—mechanisms that assume high-bandwidth, low-latency communication and consistent access to global state information. These assumptions do not hold in edge environments, where agents act asynchronously, observe only partial local states, and cannot afford frequent synchronization. As a result, centralized or communication-heavy MARL methods often suffer from degraded performance, delayed updates, or resource contention when deployed in such settings.
Many existing approaches to coordination in multi-agent reinforcement learning assume additive rewards or rely on shared incentives and communication protocols. However, in practical systems like edge computing and wireless access networks, agent interactions are tightly coupled through shared resources, and rewards are often non-additive due to congestion effects.
% Many of the works in the context of coordination in multi-agent reinforcement learning assume additive rewards or rely on shared incentives or communication protocols.
% Coordination in decentralized multi-agent reinforcement learning (MARL) has been studied in the context of potential games \cite{macglashan2022value} and communication-free learning \cite{zhang2018fully}. Many of these works assume additive rewards or rely on emergent coordination through shared incentives or communication protocols. 
% However, in practical systems such as edge computing and wireless access networks, agent interactions are tightly coupled through shared resources, and rewards are non-additive due to congestion effects. 
% Classical methods like MAPPO \cite{yu2021surprising} and IPPO \cite{de2020independent} fail to account for this structure, often leading to inefficient equilibria. 
% Recent approaches explore intrinsic motivation \cite{mckee2020social,mu2024multi} and influence-aware learning \cite{jaques2019social}, but typically rely on centralized critics or dense agent-agent communication.
To address these coordination challenges, we propose a decentralized learning framework based on independent constrained Markov decision processes (CMDPs). In this framework, coordination is achieved implicitly through shared constraints that are updated infrequently. These constraints regulate the maximum frequency at which each agent can offload tasks, effectively serving as a virtual coordination mechanism over shared resources. Each agent optimizes its own policy using techniques from constrained reinforcement learning, ensuring local autonomy while maintaining system-wide alignment through periodic, constraint-driven synchronization. This approach enables fast, scalable, and communication-efficient decision-making in distributed environments like edge computing, where centralized coordination is often impractical.

Our contributions are resumed hereafter:
\begin{itemize}
    \item We introduce DCC, a general decentralized reinforcement learning framework for multi-agent coordination under shared resource constraints; we apply it to the problem of task offloading in wireless edge computing systems. Each device (agent) solves a local CMDP, and coordination emerges implicitly through the shared constraint vector that regulates offloading behavior across the network.
    \item We provide a tractable approximation of the global objective via decomposition, establish theoretical guarantees on its validity under mild assumptions and error bounds in the nonlinear case.
    \item We validate the DCC framework through preliminary numerical experiments in toy environments. While limited in scope, our results highlight the scalability of the approach and the net improvement over centralized and independent baselines. They lay the groundwork for more extensive evaluations in future research.
\end{itemize}
% Moreover, it consistently outperforms algorithms that assume independent agents, even communication-heavy variants (IQL with common shared reward).

% This paper presents early-stage work that introduces the DCC framework and establish its theoretical basis, while validating its potential through preliminary experiments. A comprehensive empirical evaluation on realistic wireless testbeds is left to future work. % We hope this initial study will stimulate discussion and guide subsequent research on scalable, communication-efficient coordination mechanisms for wireless systems.

The paper is organized as follows: \cref{sec:related works} reviews related work, while \cref{sec:system model} introduces the Markov game and the system model under consideration. In \cref{sec:approximation optimal policy}, we present the proposed algorithm for the agents' constrained policy optimization. Numerical results are provided in \cref{sec:numerical results}, followed by concluding remarks in the final section.

\section{Related works}
\label{sec:related works}
Multi-agent reinforcement learning (MARL) has been extensively surveyed in \cite{hernandez2019survey, zhang2021multi, albrecht2024multi}. Centralized Training Distributed Execution (CTDE) approaches are the most widely used in MARL. For example, MADDPG \cite{lowe2017multi}, MAPPO \cite{yu2022surprising}, and COMA \cite{foerster2018counterfactual} employ centralized critics to guide training, while policies are executed in a decentralized manner. Value-decomposition methods such as VDN \cite{sunehag2017value} and QMIX \cite{rashid2020monotonic} improve scalability further by factorizing value functions, though they operate under an individual-global-max assumption, which does not hold in settings where agents compete for scarce resources. Independent learners such as IQL \cite{tan1993multi} and IPPO \cite{de2020independent} avoid centralization entirely, but struggle to coordinate in environments with interdependent rewards.

Constrained reinforcement learning (CRL) extends standard RL by requiring policies to satisfy long-term constraints in addition to maximizing rewards. A variety of approaches have been proposed to solve constrained Markov decision processes (CMDPs) \cite{gu2022review,liu2021policy}. Primal-dual methods, such as RCPO \cite{tessler2018reward}, optimize reward and constraints concurrently, while value-based methods \cite{bohez2019value} adapt Q-learning to the constrained setting. CPO \cite{achiam2017constrained} formulates the problem as a constrained trust-region update, offering guarantees but incurring high computational cost. IPO \cite{liu2020ipo} introduces a first-order method using barrier functions, and PCPO \cite{yang2020projection} projects policies onto the feasible set after unconstrained optimization.

% While these methods address single-agent CMDPs, they do not extend naturally to multi-agent settings with shared resource coupling. Our work differs by modeling each agent as a constrained learner and enabling decentralized coordination through shared constraints, without requiring global state or communication.

While reinforcement learning has been widely applied to MEC offloading problems—typically optimizing task latency or energy consumption using single-agent algorithms such as DQN \cite{li2018deep} and DDQN \cite{tang2020deep}—multi-agent settings remain comparatively underexplored. Some recent works integrate MARL into this domain: for example, \cite{gan2022multi} combines QMIX with DQN to jointly minimize average delay and energy cost, while \cite{huang2021multi} employs decentralized actor–critic agents for user-level offloading decisions. However, these approaches either rely on centralized training, assume frequent inter-agent communication, or do not explicitly address coordination under shared resource constraints. In contrast, our method formulates offloading as a set of independent constrained MDPs, where coordination emerges implicitly via infrequent updates to shared constraints, enabling scalable, communication-efficient learning tailored to congestible wireless environments.

\section{System model}
\label{sec:system model}
% \subsection{Task Offloading in Wireless Edge Systems}
\label{subsec: task offloading description}
In wireless edge computing, a collection of mobile devices must decide at each time step whether to execute computational tasks locally, offload them to a shared edge server, or defer processing altogether. Each device operates independently based on its local state, which may include factors such as task backlog, latency sensitivity, energy harvesting rate, local processing cost, the time elapsed since the last data processing, and battery level.
The wireless channel and the edge server represent shared, capacity-limited resources. When too many devices offload simultaneously, contention and server overload degrade performance for all users. To mitigate this, we impose a constraint on the long-term frequency with which each agent may offload, serving as a virtual coordination mechanism. This constraint {\em does not} reflect a physical device limitation, but rather a system-level policy that governs fair and efficient resource usage.

Each device is modeled as an agent solving a local constrained MDP (CMDP) \cite{altman2021constrained}, optimizing its policy to balance local performance objectives with the global constraint on offloading frequency. Coordination emerges implicitly as agents adapt their policies based on local observations and periodic updates to the shared constraint, enabling decentralized yet system-aware behavior.

\subsection{Markovian formulation}
We model the system described in \cref{subsec: task offloading description} as a decentralized multi-agent problem in which $N$ agents must coordinate their actions to minimize a global objective function. The joint state of the system is denoted by $s = (s_1, \dots, s_N) \in \mathcal{S}$, where $s_i$ is the local observation of agent $i$. Each agent selects an action $a_i$ from its individual action space, and the joint action is $a = (a_1, \dots, a_N) \in \mathcal{A}$.

We assume independent transition dynamics, i.e., the transition probability of the system factorizes across agents as $p(s' \mid s, a) = \prod_{i=1}^N p_i(s_i' \mid s_i, a_i)$, where $s' = (s_1', \dots, s_N')$ is the next state, and $p_i$ denotes the local transition kernel for agent $i$.

Each agent’s reward consists of a local component and a global component that depends on the actions of the other agents. This formulation is sufficiently general to model a broad range of mobile edge computing applications.
\begin{equation}
    r(s, a) = \sum_{i=1}^N r_i(s_i, a) = \sum_{i=1}^N u_i(s_i) + \mathbb{I}[a_i = a_{\text{crowd}}] \cdot d\left(N(a)\right)
\label{eq:immediate reward system}
\end{equation}
where $u_i: \mathcal{S}_i \to \mathbb{R}$ describes the local utility function for agent $i$, $a_{\text{crowd}} \in \mathcal{A}_i$ is a designated \textit{crowded action}, corresponding to offloading the task to the edge computer, $N(a) = \sum_{j=1}^N \mathbb{I}[a_j = a_{\text{crowd}}]$ is the number of agents choosing the crowded action and $d: [1, \infty) \to \mathbb{R}$ is a differentiable function with $d(1) = 0 , d'(n) > 0$, and $d''(n)$ exists and does not change sign.\\
The function $d(n)$ models congestion at the edge: the more agents select the crowd action, the greater the penalty each receives for choosing it. This structure induces coordination challenges reminiscent of \textit{crowding games} or \textit{singleton congestion games}, where agents compete over a shared, congestible resource. This additive form preserves the separability of local rewards while explicitly modeling the coupling from congestion through $d(N(a))$, which captures the performance degradation when many agents offload simultaneously.

We assume the following misalignment between individual and collective incentives:

\begin{assumption}[Individual Incentive for Crowd Action]
Let $a \in \mathcal{A}$ be such that $N(a) = 1$ and $a_i = a_{\text{crowd}}$, i.e., only agent $i$ selects the crowd action. Let $a' \in \mathcal{A}$ be identical to $a$ except $a_i' \neq a_{\text{crowd}}$, with $a_j' = a_j$ for all $j \neq i$. Then, for all states $s$:
\begin{equation}
r(s, a) < r(s, a').
\end{equation}
\label{assumption:individual incentive crowded action}
\end{assumption}
That is, the crowd action appears individually beneficial when selected in isolation. However, this creates a coordination dilemma: if every agent independently learns to select the crowd action, the shared penalty $d(N(a))$ increases, degrading the overall system performance. This structure highlights the need for coordination-aware learning methods beyond naive reward minimization.

The objective function we want to minimize is 
\begin{equation}
    J(\pi, \beta)  =   \expvalDist{a \sim \pi(s)}{\sum_{t=0}^\infty \gamma^t r(s_t, a_t) \mid s_0 \sim \beta}
    \label{eq:objective function}
\end{equation}
% \eqref{eq:objective function} defines the optimal expected return starting from an initial state distribution $\beta$, optimized over a joint policy $\pi$. 
While the dynamics are independent across agents, the reward function introduces a critical coupling through the crowd-dependent term $d(N(a))$, which depends on the number of agents selecting a specific action $a_{\text{crowd}}$. This structure creates several challenges in analyzing or optimizing $J(\beta)$:
\begin{itemize}
    \item \textbf{Non-decomposability:} The reward $r(s, a)$ is not additive across agents, as the congestion penalty depends on the global joint action $a$. As a result, standard decomposition techniques used in multi-agent systems with independent rewards (e.g., independent learners) are not applicable.
    \item \textbf{Coupled incentives:} Even if agents act independently, their optimal behavior is interdependent due to the shared penalty. The optimal joint policy may require implicit coordination to avoid overcrowding the sensitive action, a behavior that cannot be captured by decentralized greedy learners.
    \item \textbf{Noisy reward:} Since the congestion term depends on how many agents select $a_{\text{crowd}}$, individual agents receive noisy or misleading feedback about their own contribution to the reward. This complicates both policy gradient and value-based methods.
    \item \textbf{Non-stationarity in decentralized learning:} In decentralized training, each agent’s policy evolves independently. Coupling in the reward makes the environment non-stationary from each agent’s perspective, even though the transition dynamics are stationary.
\end{itemize}

These properties position the problem beyond standard cooperative multi-agent reinforcement learning (MARL) settings with additive rewards or shared goals. In the next section, we present a learning framework designed to address these challenges. % In particular, naive learning algorithms that do not account for the crowd-induced coupling may converge to highly suboptimal equilibria.

\section{Using MARL for optimizing offloading}
\label{sec:approximation optimal policy}
We now introduce the \textbf{DCC framework} (Decentralized Coordination via CMDPs), a general structure for multi-agent reinforcement learning in shared-resource environments. DCC enables scalable coordination among agents by combining three key ingredients:
\begin{enumerate}
    \item \textbf{Lightweight Communication:} Agents operate independently, without exchanging real-time state or action information. Inter-agent communications do occur, but only infrequently and to share scalar constraint variables, which enables coordination.
    \item \textbf{Constraint-Based Coupling:} Each agent solves a constrained MDP in which the constraint controls the crowded action. The update of the constraint variable serves as a coordination signal across agents.
    \item \textbf{3: System-Level Alignment:} Optimizing the constraint vector steers individual agent behavior toward global objectives, without requiring centralized control or synchronous communication.
\end{enumerate}

We remark that DCC is not a single algorithm but a general \emph{framework} compatible with a broad class of reinforcement learning methods for local MDPs, including value-based and policy-gradient approaches. We begin by providing a high-level intuition of the proposed algorithm. A detailed description of each component is then developed in \cref{subsec:reward approximation,subsec:CMDP single agent,subsec:finding optimal constrained policy,subsec:optimization of constraints,subsec:efficient computation of the gradient}, including the underlying reward approximation, the CMDP formulation, and the multi-timescale optimization process.

To approximate the global objective \eqref{eq:objective function}, we reformulate it as a sum of local value functions constrained by agent-specific action frequencies:
\begin{equation}
    J(\pi_{global}^\star, \beta)  \approx \inf_\theta \sum_{i=1}^N J_i(\pi^\star_i(\theta_i), \beta_i)
    \label{eq:desired approximation}
\end{equation}
Each agent $i$ solves its own CMDP, where the constraint $\theta_i$ encodes how often it can select the crowded action. This decomposition decouples learning while preserving system-level coordination through optimization of the shared constraint vector $\theta$.

To facilitate the development of a three-timescale learning algorithm, we define an approximate reward function: it replaces the actual number of agents which choose the congested action with the expected value. It corresponds to number of agents who select the crowded action according to their local constraint. This approximation removes the direct dependency on joint actions: each agent optimizes its policy independently and yet full constraint vector $\theta$ accounts for the aggregate system behavior.

Our algorithm uses a three-timescale learning process to solve the approximation in \eqref{eq:desired approximation}:
\begin{itemize}
    \item \textbf{Fast and Intermediate Timescales:} For a fixed constraint vector $\theta$, each agent independently optimizes its own policy by solving a local CMDP using a Lagrangian-based safe reinforcement learning approach. On the fastest timescale, the policy is updated using a shaped reward of the form $r_i^{\text{tot}}(s, a) = r_i(s_i, a) + \lambda_i c_i(s_i, a_i)$, where $c_i$ denotes the cost associated with selecting the crowded action. Since both $\theta$ and $\lambda_i$ are held fixed during this phase, \textit{any standard reinforcement learning algorithm}, such as PPO, DQN, or even Q-learning if the system is small enough, can be used to optimize the local policy. On the intermediate timescale, the Lagrange multiplier $\lambda_i$ is updated via a primal-dual method to enforce long-term satisfaction of the constraint.
    \item \textbf{Slow Timescale:} At the slowest time-scale, the constraint vector $\theta$ is optimized to improve global coordination. Instead of heuristically selecting $\theta$, we treat it as a coordination variable and optimize it via stochastic approximation. Importantly, the structure of the objective allows each component of the gradient $\nabla_\theta J$ to be estimated using just three local policy evaluations per agent, regardless of the total number of agents. This property ensures that the algorithm remains scalable and efficient, even in large systems.
\end{itemize}

While constraint updates are performed synchronously across agents, they occur at a much slower timescale and can be implemented with minimal coordination overhead. In practice, such updates are orders of magnitude less frequent than policy learning steps, preserving the algorithm's scalability and decentralization. Notably, the computational complexity of DCC is similar to the one of the RL algorithm chosen, with the primary difference being the computation of the gradient—an operation that is computationally inexpensive.

Together, these components yield a principled and practical algorithm for decentralized coordination under shared resource constraints, with theoretical support for its approximation quality and optimization dynamics.

% The remainder of this section provides a detailed formulation of the learning framework, starting with the construction of the local CMDPs.

\begin{algorithm}[t]
\caption{DCC Framework (Simplified Overview)}
\label{alg:dcc_simplified}
\begin{algorithmic}[1]
\State \textbf{Initialize:} constraint vector $\theta^0$, multipliers $\lambda_i^0$, policies $\pi_i^0$ for each agent $i$
\Function{OptimizeLocalCMDP}{$i, \theta, n$}
    \For{$m = 0, 1, \dots, M(n)$} \Comment{Intermediate timescale: Lagrange multiplier update}
        \For{$t = 0, 1, \dots, T(m, n)$} \Comment{Fast timescale: policy optimization}
            \State Update policy $\pi_i$ via RL step on shaped reward $r_i^{\text{tot}}$
        \EndFor
        \State Update multiplier $\lambda_i$
    \EndFor
    \State $\hat{J}_i (\theta) \gets $ evaluate final policy $\pi_i^\star$ 
    \State \Return $\hat{J}_i (\theta)$, final policy $\pi_i^\star$, multiplier $\lambda_i^\star$
\EndFunction
\\\\
\For{$n = 0, 1, 2, \dots$} \Comment{Slowest timescale: constraint optimization}
    \For{each agent $i = 1, \dots, N$}
        \State $\hat{J}_i (\theta), \pi_i^\star, \lambda_i^\star \gets$ \Call{OptimizeLocalCMDP}{$i, \theta^n, n$}
        \State Estimate local gradient $\nabla_{\theta} \hat{J}_i(\theta^n)$
    \EndFor
    \State Update shared constraint vector: $\theta^{n+1} \gets \theta^n + \eta_n \nabla \hat{J}(\theta^n)$
\EndFor
\end{algorithmic}
\end{algorithm}

A simplified pseudocode for DCC is found in \cref{alg:dcc_simplified}, while the full pseudocode in in \cref{secApp:extended algorithm}.
In the remainder of the section, we investigate more deeply the DCC framework through the following steps:
\begin{itemize}
    \item Introduction of a decomposable approximation of the immediate reward (\cref{subsec:reward approximation}):
    we propose a decomposable approximation of the immediate reward that approximates the long-term reward with provably bounded approximation error—and exact equivalence when the congestion function is linear.
    \item CMDP of the system (\cref{subsec:CMDP single agent}): we introduce the single agent CMDP that will by used by the three timescale algorithm
    \item Computing the optimal constrained policy (\cref{subsec:finding optimal constrained policy}): We describe the process of computing the optimal constrained policy for a single agent, given a fixed value of the constraint.
    \item Differentiability of the objective (\cref{subsec:optimization of constraints}): We establish that the long term reward of each agent is differentiable with respect to $\theta$, which allows us to efficiently identify the minimum in \eqref{eq:desired approximation}
    \item Efficient computation of the gradient (\cref{subsec:efficient computation of the gradient}): We show how the  structure of the gradient can be leveraged to significantly accelerate the computation of the infimum of the objective function.

\end{itemize}

% In future work, we aim to explore fully asynchronous versions of the constraint optimization step, further reducing the need for coordination and improving robustness in dynamic network settings.
\subsection{Approximating the objective function via decomposition}
\label{subsec:reward approximation}
In a general setting general, the value function decomposition is a viable alternative when
$$r(s, a) = \sum_i r_i(s_i, a_i)$$
meaning that each agent can act independently while still achieving the globally optimal outcome \cite{macglashan2022value,russell2003q}. However, in the setting studied in this work, this condition does not hold due to a coupling term that depends on how many agents select the \emph{crowded} action. Specifically, the global reward takes the form
\begin{equation}
  r(s, a = \pi_{\text{global}}(s)) = \sum_i f_i(s_i) + \mathbb{I}_{a = a_{\text{crowd}}}(a_i) \cdot d\left( N(a) \right)
  \label{eq:global reward}
\end{equation}

where \( N(a) \) denotes the number of agents that choose the crowded action under $\pi_{\text{global}}$. (as defined in \cref{sec:system model}).

To address the challenge of decomposing this reward function, we propose a new (approximated) formulation of the reward :
\begin{align}
    \hat{r}(s, a; \theta) &= \sum_i u_i(s_i) + \mathbb{I}_{a = a_{\text{crowd}}}(a_i) \cdot d\left(1 + \sum_{j \neq i} \theta_j\right) \label{eq:approximated reward} \\
    &= \sum_i \hat{r}_i(s_i, a_i; \theta_{-i}), \notag
\end{align}
where $\theta_j$ denotes the average constraint value for agent $j$, i.e. the expected frequency with which agent $j$ selects the crowded action, and $\theta_{-i}  = \sum_{j \neq i} \theta_j$.\\
Intuitively, this approximation replaces the randomness in the number of offloading agents ($d(N)$) with its expectation ($\theta$), which reduces variance in the reward signal while still capturing the average congestion effect. This makes the problem more tractable for decentralized learning, while preserving the essential trade-offs.

This amounts to treating the policies of other agents (which are unknown to any given agent) as if they select the congested action as frequently as their individual constraints permit. Although this may not strictly hold in practice, it is justified by Assumption 1, which states that the congested action yields the highest reward when chosen in isolation—implying that agents are likely to fully utilize their constraints.

This approximation brings two key benefits for decentralized learning. First, by replacing the actual number of agents choosing the crowded action with its expected value under the constraint vector $\theta$, it removes the non-stationarity typically introduced by inter-agent coupling. As a result, each agent's reward depends only on its local state, action, and the fixed parameter $\theta$, effectively decoupling the learning dynamics across agents during the fast timescale. Second, the use of an expected congestion term reduces the variance of the reward signal each agent observes, since it no longer depends on the stochastic actions of others. This lower-variance signal leads to more stable updates and can accelerate policy learning.
Finally, in practical implementations such as edge computing, this formulation allows an agent to compute its reward without waiting for real-time feedback on how many others offloaded in the same timestep. While this does not reduce communication during deployment, it simplifies both training and simulation, and avoids introducing extra delays into the reward computation.

\paragraph{Error Bound for the Reward Approximation}
The following result gives an indication of the error that one could have when considering the approximated definition of the reward in \eqref{eq:approximated reward}. 
In particular, we define $J(\pi, \beta)$ the discounted reward associated to a policy $\pi$ when considering the reward $r$ and $\hat{J}(\pi, \beta)$ the one obtained when considering the approximated reward $\hat{r}$ defined in \eqref{eq:approximated reward}. 
\begin{lemma}
\label{lemma:bound approximation reward}
    Given a global policy $\pi$ and a vector $\theta \in \mathbb{R}^N$ such that
$$
\mathbb{E}_{a \sim \pi} \left[ N_{i,t}(a) \right] = \theta_i,
$$
where $N_{i}(a)$ denotes the random variable representing the frequency with which agent $i$ selects the crowded action at time $t$, it is verified that for a non linear penalty function $d$ 
% $$\mid R_\pi(\beta) - \hat{R}_\pi(\beta) \mid \leq \frac{1}{(1-\gamma)} \rp{d(N_{agents}) - 2 d\rp{\frac{N_{agents}+1}{2}}}$$\\
$$\mid J(\pi, \beta) - \hat{J}(\pi, \beta) \mid \leq \frac{1}{1-\gamma} \sum_i \theta_i \rp{\frac{\theta_{-i}}{N_{agents}-1} d(N_{agents}) - d(1 + \theta_{-i})} $$

Moreover, if  $d$ is linear, the two reward values coincide exactly, and the approximation becomes exact.
\end{lemma}
\begin{proof}
    See \cref{subsec:proof bound approximation lemma}.
\end{proof}

This leads to the following result, which shows that, under certain conditions on the immediate reward function, the approximation in~\eqref{eq:desired approximation} holds exactly. As a result, the decomposition is valid, and the algorithm proposed in this work can reliably recover the optimal policy. Importantly, this also establishes clear conditions under which the method is guaranteed to be optimal.

\begin{proposition}
\label{prop:condition for finding exact optimal policy}
    Let assume we know $\theta^* \in \mathbb{R}^N$ such that $\expvalDist{a \sim \pi^*}{N_{i, t}(a_t)} = \theta_i^* \ \forall i$, for the optimal global policy $\pi^*$ and assume that $d(\cdot)$ is linear. Then the value function does not depend on the type of reward that we consider and an optimal global policy can be obtained as the combination of the appropriate optimal local policies.
\end{proposition}
\begin{proof}
From \cref{lemma:bound approximation reward}, we know that if the function $d$ is linear, then the discounted reward of a policy is invariant under the choice between the original reward $r$ and the approximated reward $\hat{r}$. Assuming the optimal constraint values $\theta^*$ are known, we can compute the locally optimal policy for each agent using $\hat{r}$. A policy composed by combining these local policies, denoted as $\pi^*_{\text{composed}}$, is then a solution to the global minimization problem with respect to $\hat{r}$, and by the invariance, also with respect to $r$. This leads to the conclusion:
$$J(\pi_{\text{composed}}, \beta) = \hat{J}(\pi_{\text{composed}}, \beta) = \sum_i \hat{J}_i(\pi_i^\star(\theta^*), \beta) $$
\end{proof}
% This result shows that, under certain conditions on the reward function, optimizing the global policy reduces to identifying the optimal values of the "virtual constraints" $\theta$, followed by computing the corresponding local constrained policies.

\subsection{Constrained MDP for Each Agent}
\label{subsec:CMDP single agent}

As introduced in \cref{eq:desired approximation}, we approximate the global objective by decomposing it into a collection of independent Constrained Markov Decision Processes (CMDPs), one per agent. Constrained Markov Decision Process (CMDP) \cite{altman2021constrained} extends the standard MDP by incorporating constraints on long-term cost signals. In this section, we describe the structure of the CMDP associated with each agent $i$.

\paragraph{State space}
The state space $\mathcal{S}_i$ for agent $i$ consists of its local observation $s_i$, such that the full joint state is $(s_1, \dots, s_N) \in \mathcal{S}$.

In the task offloading scenario considered, each agent's state is given by $s_i = (x_i, e_i)$, where $x_i$ denotes the \emph{Age of Information}, i.e., the number of timesteps since the last data processing was completed, and $e_i$ represents the current energy level of the agent's battery.

\paragraph{Action space}
The action space $\mathcal{A}_i(s_i)$ includes a designated \emph{crowded action}, $a_{\text{crowd}}$, which represents use of a shared resource. All other actions are independent of shared usage. 

For example, in the task offloading scenario described in \cref{sec:system model}, we have $\mathcal{A}_i(s_i) = \{ \text{``wait"}, \text{``local processing"}, \text{``offload"} \}$, where ``offload" corresponds to $a_{\text{crowd}}$.

\paragraph{Reward function}
The agent receives a reward according to the approximated formulation described in \cref{subsec:reward approximation}, where the influence of other agents is captured through the fixed (in the slow and intermediate timescale) parameter $\theta_{-i}$:
\[
\hat{r}_i(s, a; \theta_{-i}) = u_i(s_i) + \mathbb{I}_{a_i = a_{\text{crowd}}} \cdot d\left(1 + \theta_{-i}\right).
\]

\paragraph{Cost function}
The cost signal $c_i(s_i, a_i)$ is a naive function active only when the shared resource is used:
\begin{equation}
    c_i (s_i, a_i) = \begin{cases*}
        1, & if $a_i = a_{\text{crowd}}$ \\
        0, & otherwise
    \end{cases*}
\end{equation}

\paragraph{Constraint}
The constraint is defined by the parameter $\theta_i$, which limits the long-term frequency with which agent $i$ can select the crowded action.

\vspace{1em}
While our formulation is based on the discounted reward criterion, the DCC-RL framework is not inherently restricted to this setting. In principle, the same decomposition and coordination structure can be applied in the average reward case, with appropriate adjustments to the underlying reinforcement learning algorithm. The choice between average and discounted formulations depends primarily on the application context and the algorithm used to solve each agent’s constrained MDP.  In the discounted case, we define the discounted reward and discounted cost under policy $\pi_i$ and initial state distribution $\beta_i$ as follows:
\begin{align}
    J_i(\pi_i, \beta_i; \theta_{-i}) &= \mathbb{E}_{\pi,\, s_0 \sim \beta} \left[ \sum_{t=0}^\infty \gamma^t \sum_{t=0}^T \hat{r}_i(S_t, A_t; \theta_{-i}) \,\middle|\, S_0 = s_0 \right] \label{eq:discounted reward CMDP} \\
    K_i(\pi_i, \beta_i) &= \mathbb{E}_{\pi,\, s_0 \sim \beta} \left[ \sum_{t=0}^\infty \gamma^t c_i(S_t, A_t) \,\middle|\, S_0 = s_0 \right] \label{eq:discounted cost function CMDP}
\end{align}

The optimization problem faced by each agent, assuming fixed constraint parameters $\theta$, is then:
\begin{align}
    \text{minimize:} \quad & J_i(\pi_i, \beta_i; \theta_{-i}) \tag{$CMDP_i$} \label{eq:CMDP agent i} \\
    \text{subject to:} \quad & K_i(\pi_i, \beta_i) \leq \theta_i \notag
\end{align}

\subsection{Learning the Optimal Local Constrained Policy}
\label{subsec:finding optimal constrained policy}

A central advantage of our framework is its flexibility: the local constrained policy optimization can be performed using any safe reinforcement learning (SRL) method suited to the environment. In particular, by adopting a Lagrangian-based approach, we decouple constraint enforcement from reward maximization. This separation enables the use of any standard reinforcement learning algorithm on the fastest timescale, with constraint satisfaction handled independently on a slower timescale. This modularity is crucial in decentralized settings with heterogeneous agents or dynamics, allowing DCC to adapt to a wide range of learning scenarios and observation structures.

Lagrangian methods are a natural fit for our setting, as they allow constraint satisfaction to be enforced without the need for projection or repair, and are compatible with standard RL algorithms. While such methods can exhibit instability or slow convergence—particularly in settings where constraint violations incur severe penalties or pose safety risks—these issues are mitigated in our case. Here, constraints reflect system-level efficiency or coordination objectives rather than hard safety requirements, so temporary violations during training do not lead to catastrophic outcomes. As a result, we can safely use Lagrangian updates without sacrificing performance or stability, making them a principled and practical choice for decentralized learning under shared constraints.

In particular, we follow the structure of RCPO \cite{tessler2018reward} for constrained optimization, employing either PPO \cite{schulman2017proximal}, DQN \cite{mnih2013playing} or Q-learning \cite{watkins1992q}, depending on the scale of the system, as the algorithm to handle the fastest timescale

In this framework, in the fastest timescale each RL agent optimizes a shaped reward of the form:
\begin{equation}
    \hat{r}_i^{\text{total}} (s_i, a_i; \theta_{-i}) = \hat{r}_i (s_i, a_i; \theta_{-i}) + \lambda_i c_i (s_i, a_i),
    \label{eq:shaped_reward}
\end{equation}
where $\hat{r_i}$ and $c_i$ have been defined in \cref{subsec:CMDP single agent}. \\
% This aligns with prior works in the context of MARL applied to social dilemmas \af{(here it would be nice to have a sentence to describe social dilemmas, otherwise eliminate the reference)}, where the reward is usually reshaped as $$r_i^{\text{total}} = \alpha r_i^{\text{env}} + \beta r_i^{\text{mot}}$$ with $\alpha$ and $\beta$ representing weights which can be defined in different ways \cite{hughes2018inequity,mckee2020social,mu2024multi,wang2018evolving}, \( r_i^{\text{env}} \) is the standard environmental reward \cref{eq:immediate reward system} , and \( r_i^{\text{mot}} \) is an intrinsic motivation reward derived from a shared constraint mechanism and acts as the cost function in our application. \\

The scalar coefficient \( \lambda_i \) serves as a Lagrange multiplier, adaptively tuned during training to ensure constraint satisfaction. To formalize this, we define the agent-specific Lagrangian function:
\begin{equation}
    \mathcal{L}_i(\pi_i, \lambda_i; \theta) = \mathbb{E}_{\pi_i} \bigg[ \sum_{t=0}^\infty \gamma^t \hat{r}_i(s_t, a_t; \theta_{-i}) + \lambda_i \left( c_i(s_t, a_t) - \theta_i \right) \bigg]
\label{eq:lagrangian function}
\end{equation}
where \( \theta_i \) is the fixed constraint threshold for agent $i$, and $\theta_{-i}$ denotes the sum of the vector of constraint parameters associated with all other agents. In this formulation, $\theta_{-i}$ is treated as fixed during the agent's local optimization, capturing the expected influence of other agents on the congestion dynamics.
The $k$-th update of the Lagrange multiplier is then computed as follows:
\begin{align*}
    \lambda_i^{k+1} &= \Gamma_{\lambda_i} \left[ \lambda_i^k + \eta_i(k) \nabla_\lambda \mathcal{L}_i (\pi_i, \lambda_i^k ; \theta)   \right] \\
    &= \Gamma_{\lambda_i} \left[ \lambda_i^k + \eta_i(k) \rp{\expvalDist{\pi_i}{\sum_{t=0}^\infty \gamma^t \sum_{t=0}^T c_i} - \theta_i } \right]
\end{align*}
and $\Gamma_{\lambda_i}$ is a projection on the space of possible Lagrange multipliers, i.e. the set of positive real numbers and $\eta_i(t)$ is a descending sequence converging to 0.

Crucially the value of \( \lambda_i \) is optimized for each agent independently, allowing the coordination signal to emerge in a fully decentralized manner. Our method enables agents to learn how the cost function should influence their behavior individually, in order to satisfy long-term constraints while maintaining local performance.

\subsection{Coordination via optimization of the constraints}
\label{subsec:optimization of constraints}
Having described the two faster timescales, we now focus on the slowest one, where the constraint vector $\theta$ is optimized to improve global coordination.

In particular, we want to show that we can converge to the local optimal value of the constraint vector $\theta \in \mathbb{R^N}$. The safe RL algorithms described in \cref{subsec:finding optimal constrained policy} will then be able to find the optimal policy for that value.

We consider the long term Lagrangian reward with regard to $\theta$, assuming that we can solve the corresponding constrained problem and find the optimal values of $\pi$ and $\lambda$. This corresponds to the quantity we want to minimize for each agent. With abuse of notation, we define the quantity 
%We consider the long term Lagrangian reward of each agent, $\tilde{J}^\ell_i(\theta_i, \lambda_i, \pi_i)$, defined as 
\begin{align}
    \hat{J}_i(\theta) &= \hat{J} (\pi_i^\star (\theta))\notag \\
    & =  \mathcal{L}_i (\pi_i^\star (\theta), \lambda_i^\star (\theta) ; \theta) \notag \\
    &= \mathbb{E}_{a \sim \pi^\star_i(\theta)} \bigg[ \sum_{t=0}^\infty \gamma^t \sum_{t=0}^T \hat{r}_i(s_t, a_t; \theta_{-i}) + \lambda_i^\star \cdot (c_i(s_t, a_t) - \theta_i) \bigg] \label{eq:definition objective function optimal policy}
\end{align}  
and the function we want to minimize is 
\begin{equation}
    \hat{J}(\theta) = \sum_i \hat{J}_i (\theta)
    \label{eq:objective function slow timescale}
\end{equation}
First we show the differentiability of this function with regard to the constraint $\theta$.
\begin{lemma}
    $\hat{J}_i(\theta)$ is differentiable in $\theta$ almost everywhere.
    \label{lemma:component objective function differentiable}
\end{lemma}
\begin{proof}
    The proof of this result is provided in \cref{subsec:proof differentiability objective function}, where we show that a variation of the Envelope Theorem \cite{milgrom2002envelope} can be applied within the context of our setting.
\end{proof}

% \af{here we should briefly discuss the fact that we can use stochastic approximation. And then in the final subsection, mention proposition 2 as a way to compute exactly the gradients given only one evaluation and simplify everything (still mention that it is an approximation, so we could say it is a way to simplify the computation of the derivatives)}

In order to discuss the optimization of the virtual constraint we want to use the stochastic gradient ascent methods of the Kiefer-Wolfowitz family \cite{kushner1978stochastic}. The iteration scheme is
\begin{equation}
    \theta^{n+1} = \Pi_{\Theta} \rp{\theta^n  - \alpha_n \hat{g}_n}
    \label{eq:iterative scheme stochastic approximation}
\end{equation} where $\theta^n$ is the $n$th iterate of the parameter, $\hat{g}_n$ represents an estimate of the gradient of the objective function, $\{\alpha_n\}_n$ is a sequence converging to 0 and $\Pi_{\Theta}$ is a projection on the space of possible virtual constraint vectors $\Theta$.

If we drop $\lambda$ and $\pi$ for the sake of clarity, the $i$-th component of the gradient estimate writes 
\begin{equation}
    \rp{\hat{g}_n}_i = \frac{\hat{J} \rp{\theta + c_n \rp{\Delta_n}_i} - \hat{J} \rp{ \theta}}{c_n (\Delta_n)_i}
\end{equation}
where $\{c_n\}$ is a sequence converging to 0 and  $\{ \Delta_n \}$ is an i.i.d. vector sequence of perturbations of i.i.d. components $\{ (\Delta_n)_i, i = 1, \dots, N \}$ with zero mean and where $\expvalDist{}{\lvert (\Delta_n)_i^{-2} \rvert}$ is uniformly bounded.

The following result defines the conditions on the objective function, step-size sequence ${\alpha_n}$, and gradient estimates $\hat{g}_n$ that give the convergence to a local minima.
\begin{comment}
\label{prop:convergence SPSA problem studied}
     Let $\{ \alpha_n \}$ and $\{ c_n \}$ be such that $\sum_n \alpha_n = \infty$, $\sum_n \rp{\frac{\alpha_n}{c_n}}^2 < \infty$. Moreover, assume that $\expval{\lvert (\Delta_n)_i \rvert^{-2}}$ is uniformly bounded on $\Theta$ and that $\tilde{J}$ is strictly unimodal over $\Theta$. Then the iteration  \eqref{eq:iterative scheme stochastic approximation} converges to the optimal parameter  $\theta^*$  w.p.1.
% \end{proposition}
% \begin{proof}
    See Proposition 1 in \cite{l1994stochastic}.
\end{comment}
\begin{theorem}
\label{thm:convergence algorithm}
Assume that 
\begin{itemize}
    \item $\{ \alpha_n \}$ and $\{ c_n \}$ be such that $\sum_n \alpha_n = \infty$, $\sum_n \rp{\frac{\alpha_n}{c_n}}^2 < \infty$
    \item $\expvalDist{}{\lvert (\Delta_n)_i \rvert^{-2}}$ is uniformly bounded on $\Theta$
\end{itemize}
Then the algorithm converges to a local optimal value of $\min_\theta \sum_i \hat{J}_i(\theta)$.
\end{theorem}
\begin{proof}
    \Cref{lemma:component objective function differentiable} guarantees the differentiability of the objective function. The assumptions on the sequences $\alpha_n$ and $c_n$ allow us to use Proposition 1 in \cite{l1994stochastic} for a subset of $\Theta$ which contains the initial vector $\theta_0$. This proves the convergence of the algorithm to a locally optimal solution.
\end{proof}

See Proposition 1 in \cite{l1994stochastic} for the statement and the necessary conditions for the convergence to the global optimal point.

\subsection{Numerical optimization of the constraints via evaluations of the local policies}
\label{subsec:efficient computation of the gradient}

We conclude this section by showing how the structure of the gradient can be further simplified, leading to a more efficient computation.

% To this end, we observe that each component of the approximate objective function can be written as $$\tilde{J}^\ell_i(\theta) = \tilde{J}^\ell (\theta_i, \theta_{-i})$$

The following result, which is an immediate consequence of the chain rule,  reveals how the computation of the gradient of the objective function can be simplified, thereby improving the overall efficiency of the algorithm:
\begin{proposition}
    \begin{equation}
        \frac{\partial  }{\partial \theta_i} \hat{J}(\theta_i, \theta_{-i}) = \frac{\partial \hat{J}_i (\theta_i, \theta_{-i})}{\partial \theta_i}  + \sum_{j \neq i} \frac{\partial \ \hat{J}_j(\theta_i, \theta_{-i})}{\partial \theta_{-j}}
        \label{eq:derivative objective function}
    \end{equation}   
    \label{prop:decomposition of the objective function}
\end{proposition}
\begin{proof}
    After applying the chain rule, it is sufficient to observe that $ \frac{\partial \theta_{-j}}{\partial \theta_i} =1$ when $j\neq i$.
\end{proof}    

\Cref{prop:decomposition of the objective function} illustrates how the problem structure, together with the form of the constrained value function, can be leveraged to compute the gradient of the approximate objective function efficiently. Specifically, this result allows each component of the gradient at a given point $\theta$ to be estimated using only three stochastic evaluations per agent, rather than $N+1$, as we would only need to evaluate $\hat{J}_i(\theta_i, \theta_{-i}), \hat{J}_i(\theta_i + \epsilon, \theta_{-i})$ and $\hat{J}_i(\theta_i, \theta_{-i} + \epsilon)$.
% In this context, computing $\tilde{J}^\ell_i(\hat{\theta}_i, \hat{\theta}_{-i})$ involves finding the optimal constrained policy for agent $i$ under the constraint vector $\hat{\theta}$ using the safe reinforcement learning algorithm described in \cref{subsec:finding optimal constrained policy}, and then evaluating the resulting policy.\\
% Building on \Cref{prop:sign derivatives objective function}, the derivative of the objective could be computed directly, avoiding additional evaluations of the optimal noisy policy and thus improving efficiency.  We leave this refinement to future work.

While \Cref{prop:decomposition of the objective function} enables an efficient numerical estimation of the gradient, the following result provides a more precise analytical characterization of its components and clarifies their qualitative behavior.

% The following result formalizes the intuition regarding the sign of the derivatives of the objective function and express the exact form of the local gradients. In particular, it shows that relaxing the constraint for one device decreases its own reward (i.e., benefits it from a selfish perspective) while penalizing the others. This behavior arises naturally from the problem’s structure: if one device is allowed to offload more frequently, it experiences a lower individual reward, whereas the remaining devices must offload less often due to the increased congestion penalty. %We also numerically verified these results in \cref{subsec:appendix gradient objective function}, using linear programming to compute the optimal policy in small systems.
\begin{proposition}
The following is verified:
    \begin{enumerate}
        \item $\frac{\partial \hat{J}_i (\theta)}{\partial \theta_i} = - \frac{\lambda^\star}{M}  \leq 0$ a.e., where $\lambda^\star$ is the optimal value of the Lagrange multiplier in $\theta_i$ and $M$ is a positive constant,
        % \item $\frac{\partial \hat{J}_i (\theta)}{\partial \theta_i} \leq 0$ a.e., 
        \item $\frac{\partial \hat{J}_i (\theta)}{\partial \theta_j} = \theta_i  \frac{\partial}{\partial \theta_j} d(1 + \theta_{-i}) \geq 0, \ \forall j \neq i$ a.e., where $\pi^\star_i$ is the optimal policy for agent $i$.
        % \item $\frac{\partial \hat{J}_i (\theta)}{\partial \theta_j} \geq 0, \ \forall j \neq i$ a.e.
    \end{enumerate}
    \label{prop:sign derivatives objective function}
\end{proposition}
\begin{proof}
    It is once again a consequence of \cref{thm:envelope theorem constrained problem}; in particular of the definition of the directional derivatives. The full proof is provided in \cref{subsec: proof sign derivatives}. 
\end{proof}

\Cref{prop:sign derivatives objective function} formalizes the intuition about the sign of the gradient components and provides the exact expression for the local derivatives. In particular, relaxing the constraint for one device decreases its own reward (i.e., benefits it from a selfish perspective) while penalizing the others, as a consequence of the shared congestion effect.  

It is worth noting, however, that this expression still relies on the exact values of the optimal multipliers $\lambda_i^\star$, which are not directly available during learning. The proposed algorithm converges to these values asymptotically, but in practice the gradients are approximated. For this reason, in the numerical experiments we currently rely on finite-difference estimates.  

The additional simplification obtained by directly substituting $\lambda_i^\star$ in place of the local finite-difference evaluations could further reduce computation times. This refinement has not yet been tested in the stochastic approximation setting, though we have verified its correctness in small systems where $\lambda_i^\star$ was computed exactly via linear programming (\cref{subsec:appendix gradient objective function}). Evaluating the performance of this direct substitution in the approximate setting is left for future work.

\medskip

The framework described in this section provides a principled and scalable approach to decentralized coordination in multi-agent systems with limited communication. By decomposing the global objective through constrained MDP formulations and optimizing across three timescales, agents can independently learn behaviors that align with system-level goals. The theoretical guarantees and structural properties of the reward function justify this decomposition and enable efficient learning despite the inherent coupling in agent behavior. These results lay the foundation for the empirical evaluation in the next section, where we assess the algorithm’s performance under realistic edge computing scenarios.

\section{Numerical experiments}
\label{sec:numerical results}
We consider the environment introduced in \cref{sec:system model}, with the toy model detailed in \cref{sec:appendix toy model description}. Each experiment is averaged over 15 runs on independently generated environments. % to improve robustness. 
We focus on small instances where the optimal policy can be determined via Q-learning, enabling clear validation of our approach.

We compare \icAlg against two representative baselines. First, independent Q-learning (IQL), which represents the fully decentralized setting without communication. Second, MAPPO \cite{yu2022surprising}, a CTDE method that has shown strong performance in cooperative MARL benchmarks. We use IQL rather than the more common IPPO because the toy environments are very small, and IPPO’s larger architectures tend to overfit, obscuring the relative benefits of our coordination mechanism. % Additional comparisons with a centralized actor–critic (A2C) and with the shared-reward variant of IQL are reported in the appendix.

% Baselines are evaluated every 100k episodes, while \icAlg is evaluated every 300k episodes due to the three evaluations required for its gradient computation (\cref{sec:approximation optimal policy}). Each episode spans 100 steps, and we report discounted returns. The flexibility of the DCC framework also allows for an average-reward formulation, which we leave for future exploration.

Finally, we stress that this is an early-stage exploration of the DCC framework. Our experiments on toy environments are intended as proof-of-concept to validate theoretical properties and highlight scalability trends, rather than as comprehensive benchmarks. Extending this evaluation to realistic wireless network simulators constitutes an important direction for future work.
% For the smaller systems with 10 devices we also add an evaluation obtained applying the linear program assuming fixed uniform constraint $(\theta_i = 1/N \ \forall \ i)$.
\begin{figure}
    \begin{subfigure}[t]{0.55\textwidth}
        \centering
        \includegraphics[width = \linewidth]{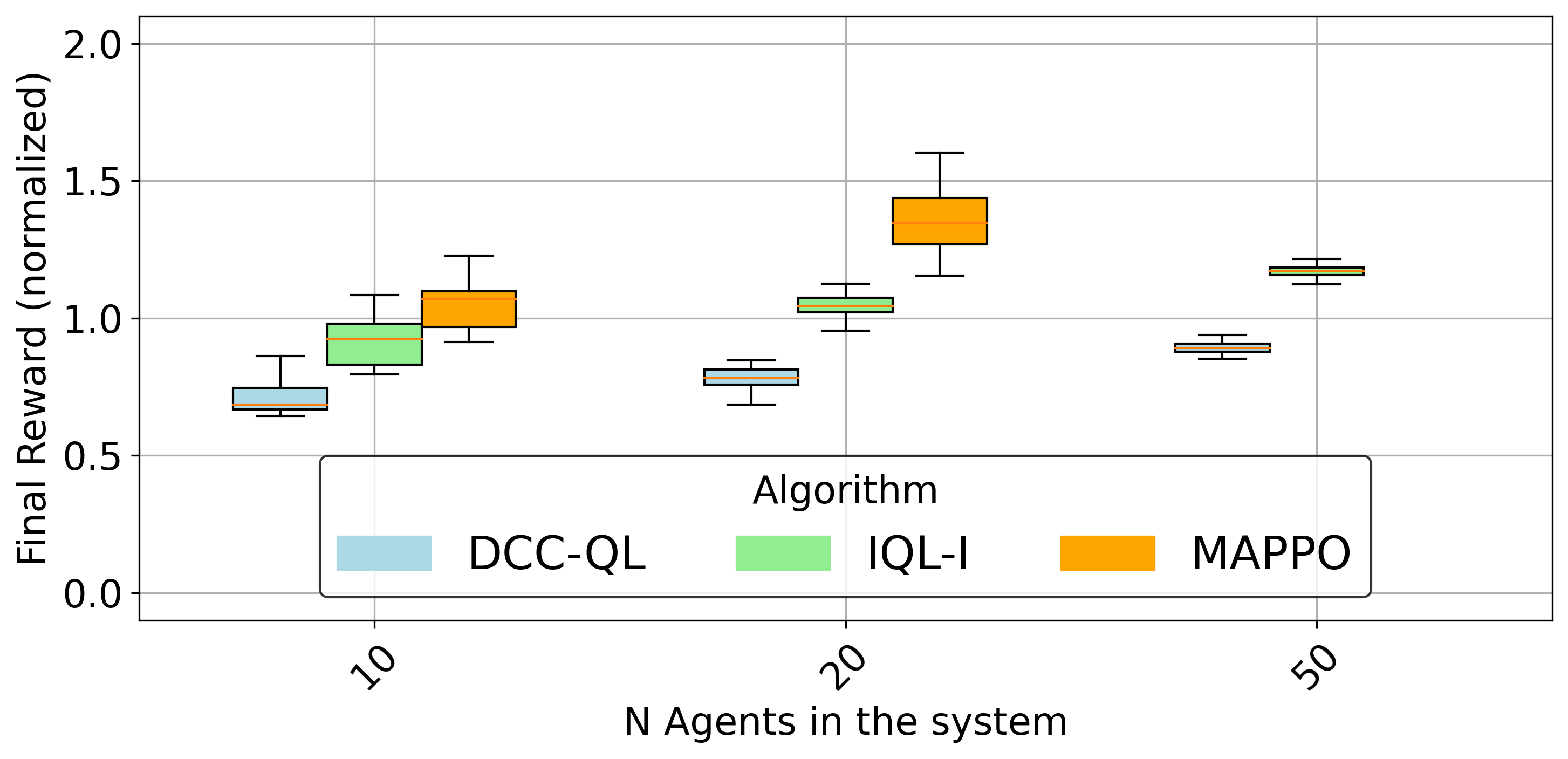}
        \label{subfig:scalability comparison}
        \caption{Normalized final average reward across methods and system sizes. Rewards are scaled so that 1 equals \icAlg’s performance after the first $10^5$ steps. Error bars indicate variability across seeds. MAPPO results for 50 devices (average $\approx  3.5$) are omitted to avoid distorting the scale.}
    \end{subfigure}
    \hfill
    \begin{subfigure}[t]{0.39\textwidth}
        \centering
        \includegraphics[width = \linewidth]{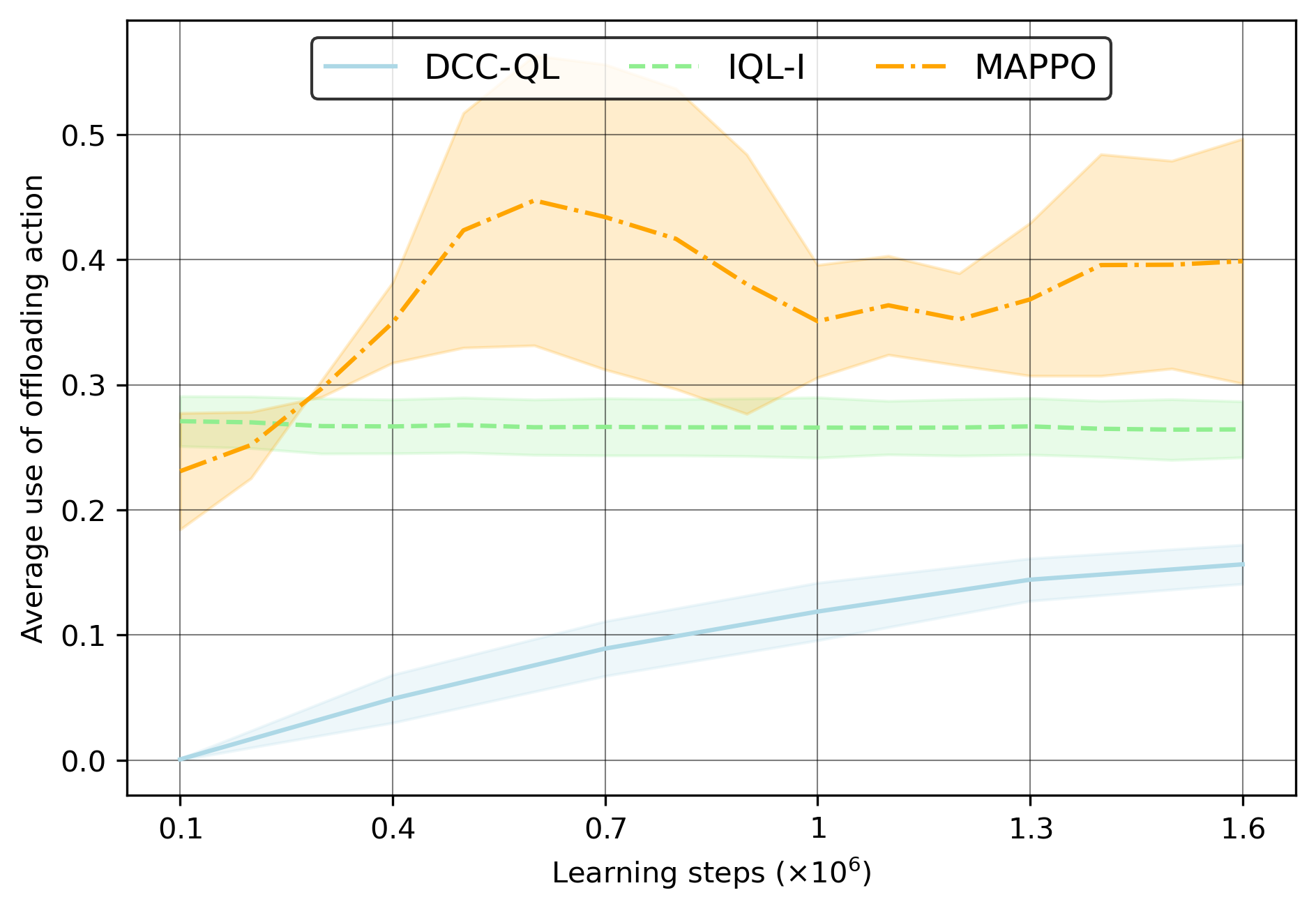}
        \caption{Evolution of the offloading frequency in a 10-device system. Starting from $\theta = 0$, \icAlg gradually converges to stable usage, while IQL quickly locks into a suboptimal policy with excessive offloading.}
        \label{subfig:offloading action comparison}
    \end{subfigure}
    \label{fig:main numerical results}
\end{figure}

\subsection{Scalability properties}
The first experiment evaluates the scalability of \icAlg by comparing its discounted reward after five iterations of constraint improvement—starting from $\theta = 0$, meaning agents are initially prohibited from using the common resource—with that of the baselines as the number of devices increases. All methods are given the same total number of policy-learning steps, and rewards are normalized so that 1 corresponds to the reward achieved by \icAlg after the first $10^5$ learning steps. As shown in Fig. a, \icAlg consistently outperforms independent Q-learning across all system sizes, while MAPPO, though competitive in small systems, degrades rapidly as the number of devices grows, likely due to its fixed network architecture being unable to handle the enlarged state–action space.
We also note that in larger systems, never using the offloading action—i.e., the normalization baseline obtained with $\theta = 0$ —yields higher rewards than IQL, underscoring how IQL’s lack of coordination prevents it from capturing the system dynamics.

\subsection{Offloading action frequency}
As a second observation, we analyze the frequency of the offloading action in the case with 10 devices (results for $N=20$ and $N=50$ are provided in the additional material). Figure b shows that \icAlg gradually evolves from never using the offloading action (the initial evaluation and normalization baseline) to a moderate and stable usage level, consistent with the constraints derived in \cref{sec:approximation optimal policy}. In contrast, IQL quickly converges to a suboptimal policy that overuses the offloading action, as expected given its attractiveness (Assumption 1) and the lack of any coordination mechanism. MAPPO, meanwhile, has not converged within the limited training budget considered here and continues to overuse the offloading action; however, when trained for 10 million steps (six times more than evaluated in this experiment), it eventually achieves performance comparable to \icAlg, indicating that it can learn the dynamics of the small system but only at significantly higher sample complexity.

\subsection{Effect of Initial Constraint}

In the main experiments, the environments considered were relatively homogeneous. In such cases, starting from uniform constraints is a natural choice. We therefore repeated the first main experiment from \cref{sec:numerical results}, initializing \icAlg with larger values of the constraints instead of $\theta = 0$. In particular, we chose as initial constraints values similar to the final values obtained by \icAlg from naive initial constraints.
\begin{figure}[h]
    \centering
    \includegraphics[width=0.65\linewidth]{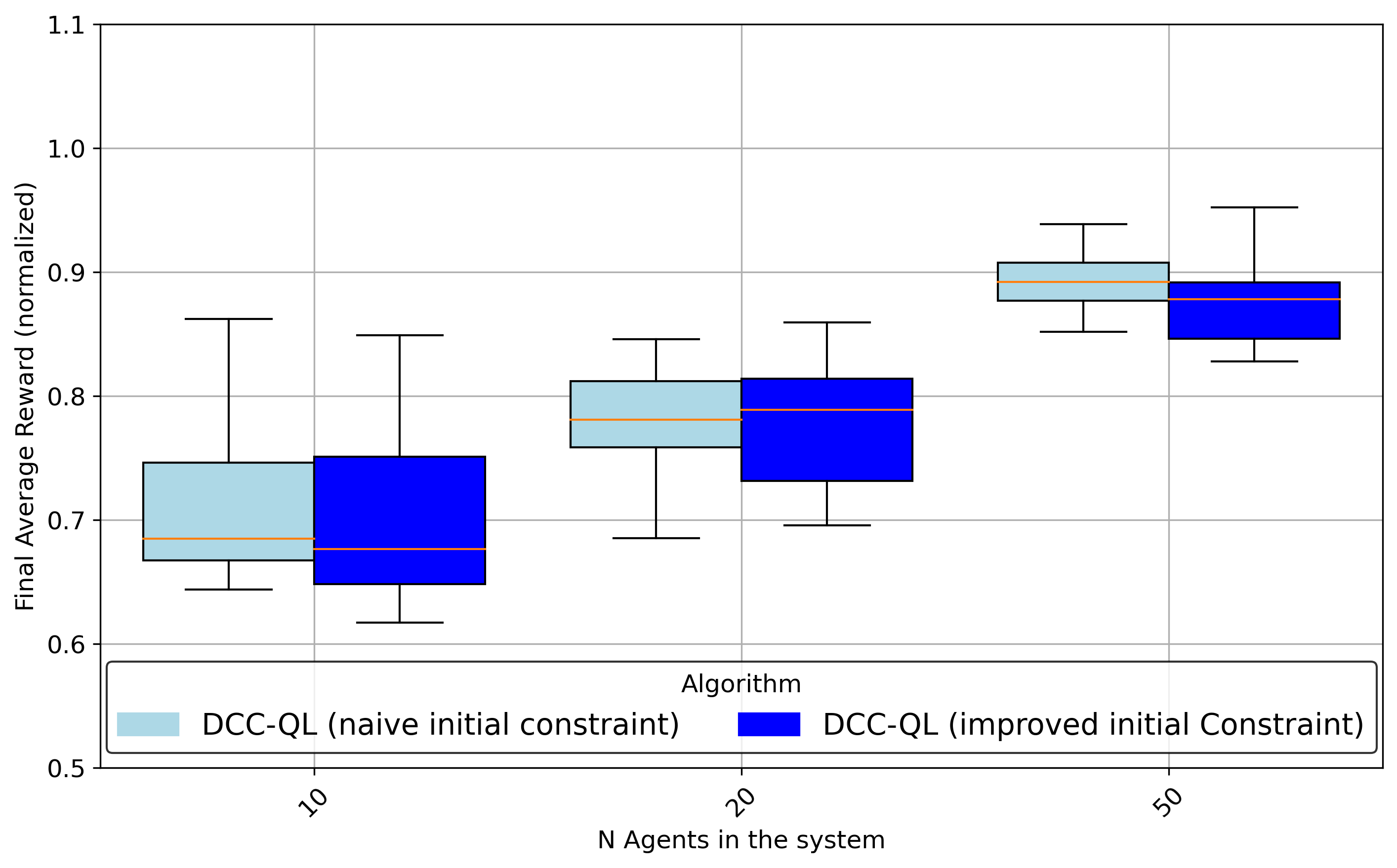}
    \caption{We compare the final normalized reward after 5 iteration of \icAlg when starting from a naive constraint ($\theta = 0$) and when starting from an optimized value, where the optimized one has been chosen by looking at the final values obtained by \icAlg when starting from naive initial constraints.}
    \label{fig:uniform constraint comparison boxplot}
\end{figure}

As shown in \cref{fig:uniform constraint comparison boxplot}, after five iterations of constraint improvement the final discounted reward is essentially the same for both initializations, with uniform initialization yielding slightly higher values. The difference is more pronounced in early iterations: when starting from uniform constraints, the algorithm achieves higher rewards sooner, particularly in systems with fewer devices. This confirms that a better prior on the constraints can accelerate learning, as observed in \cref{fig:comparison evolution reward uniform initial}.
\begin{figure}
    \centering
    \includegraphics[width=1\linewidth]{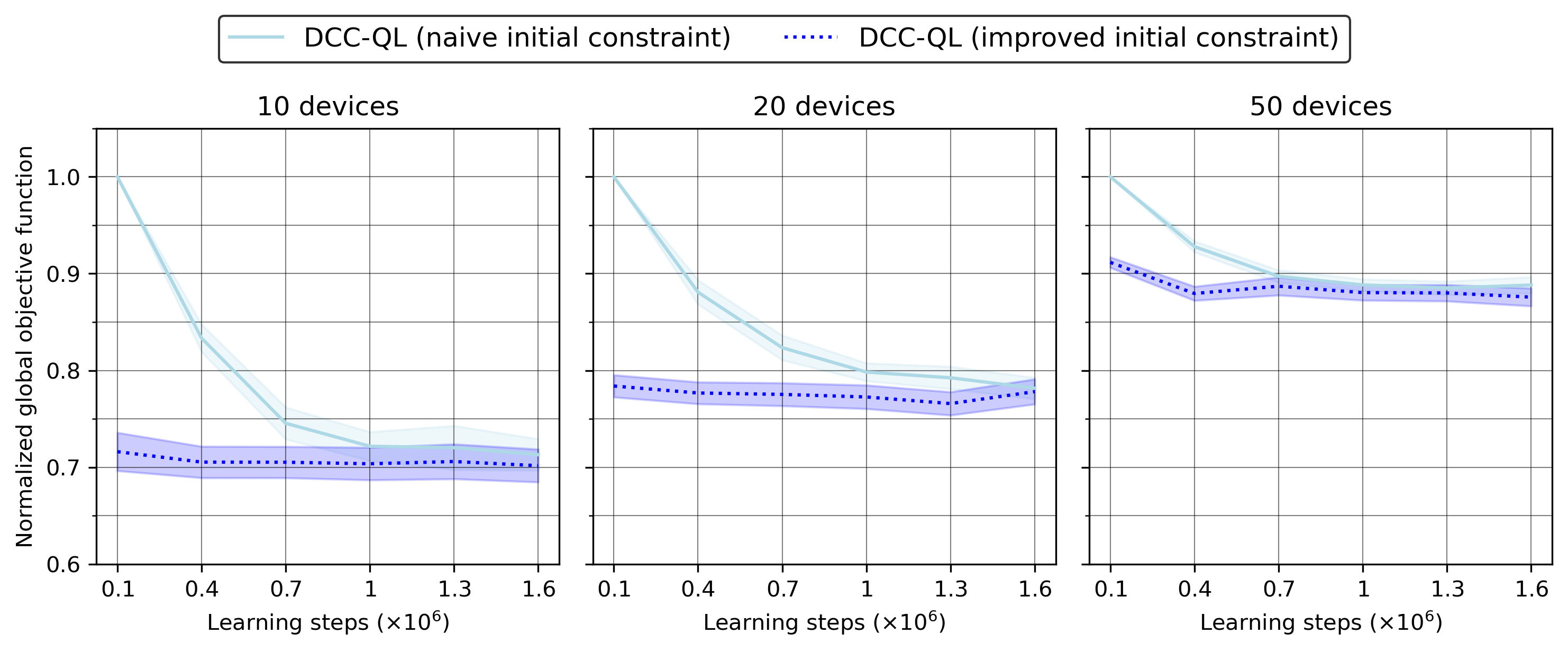}
    \caption{Comparison of the evolution of the reward as we start from optimized initial constraints in settings with a different amount of devices.}
    \label{fig:comparison evolution reward uniform initial}
\end{figure}

\subsection{Non linear penalty}
We investigate how the frequency of the offloading action changes when varying the exponent $\alpha$ of the penalty function $d(n)=(n-1)^\alpha$. The experiment was carried out with $20$ devices, since in the case of $10$ devices the offloading frequency remains essentially unchanged across all values of $\alpha$. This is likely because, with fewer devices, it is rare for many of them to simultaneously use the shared resource, making the penalty less pronounced. Moreover, the penalty is independent of $\alpha$ when exactly two devices offload at the same time, so differences only emerge when three or more devices compete for the shared resource—a situation that occurs more frequently with $20$ devices. \\

\begin{figure}[h]
    \centering
    \includegraphics[width=\linewidth]{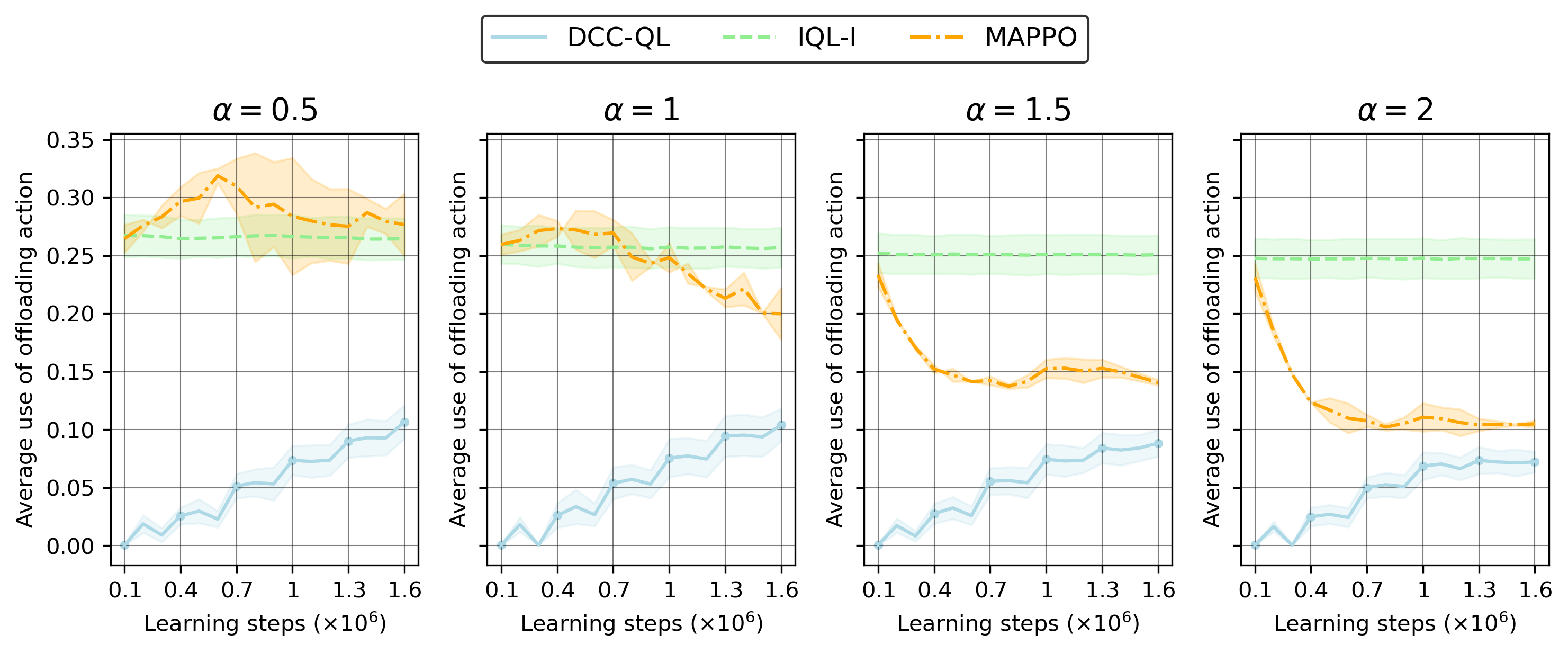}
    \caption{Evolution of the offloading action frequency for different values of the penalty exponent $\alpha$. Increasing $\alpha$ leads to a consistent decrease in offloading frequency across algorithms, reflecting the stronger penalization of simultaneous offloading.}
    \label{fig:comparison offloading nonilnear penalty}
\end{figure}
\Cref{fig:comparison offloading nonilnear penalty} shows that, across all algorithms, the offloading frequency decreases as $\alpha$ increases, consistent with the expectation that higher exponents strengthen the penalty and discourage simultaneous offloading. The effect is particularly pronounced for \icAlg and MAPPO. 

\subsection{Evaluation gradient objective function}
\label{subsec:appendix gradient objective function}
We conducted an additional experiment to empirically validate the gradient properties derived in \Cref{prop:sign derivatives objective function}. For each value of the exponent of the penalty function, we generated 15 small random environments and solved each CMDP with the linear program to obtain the optimal policy and its discounted reward.
First we perturbed the constraint vector by a finite noise $\epsilon$ and estimated gradients with respect to both the local component $\theta_i$ and the coupling term $\theta_{-i}$. The results (Fig.~\ref{fig:approximate gradient objective function}) confirm the theoretical prediction and the intuition: the local gradient is consistently negative, while the coupling gradient is positive.\\
\begin{figure}[h!]
\begin{comment}
    \begin{subfigure}[t]{0.45\textwidth}
        \centering
        \includegraphics[width = \linewidth]{figures/approximate_gradient_vs_noise_local_component.png}
        \caption{Gradient with regard to $\theta_i$; we expect only negative values according to \cref{prop:sign derivatives objective function}.}
        \label{subfig:local approximate gradient component}
    \end{subfigure}
    \hfill
    \begin{subfigure}[t]{0.45\textwidth}
        \centering
        \includegraphics[width = \linewidth]{figures/approximate_gradient_vs_noise_coupling_component.png}
        \caption{Gradient with regard to $\theta_{-i}$; we expect only positive values according to \cref{prop:sign derivatives objective function}.}
        \label{subfig:coupling approximate gradient component}
    \end{subfigure}
\end{comment}
    \centering
    \includegraphics[width = .8 \linewidth]{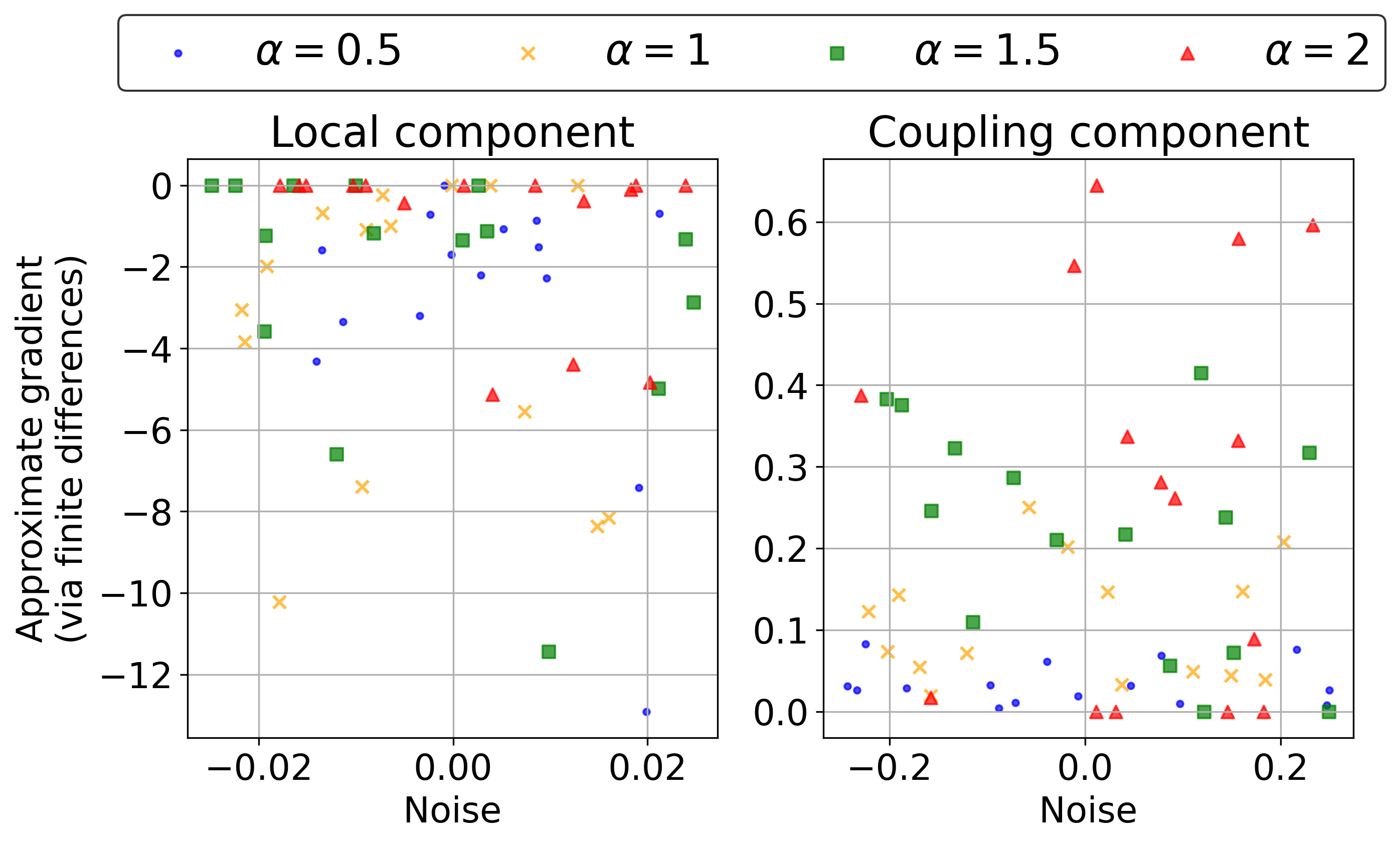}
    \caption{In these figures we evaluate an approximation of the gradient of $\tilde{J}_i^\ell(\theta)$ using the finite difference method. We considered a noise $\epsilon \in (0.01, 0.25)$ for both cases, In the left figure, representing the local component of the noise, we expected negative values, while in the right figure we expected positive values.}
    \label{fig:approximate gradient objective function}
\end{figure}
Then, using very small $\epsilon$, we compared left and right finite-difference estimates, which matched closely, confirming differentiability (not reported). Moreover, in \cref{fig:verification gradient exact values} we verified that the local derivative equals the negative of the optimal Lagrange multiplier, and that the coupling derivative matches the expression in \Cref{prop:sign derivatives objective function}.

These findings validate \cref{prop:sign derivatives objective function} and suggest a possible refinement of \icAlg: instead of estimating local gradients via finite differences, the algorithm could directly substitute the Lagrange multiplier, potentially leading to a more efficient implementation. We leave this extension to future work.

\begin{figure}[t]
    \begin{subfigure}[t]{0.45\textwidth}
        \centering
        \includegraphics[width = \linewidth]{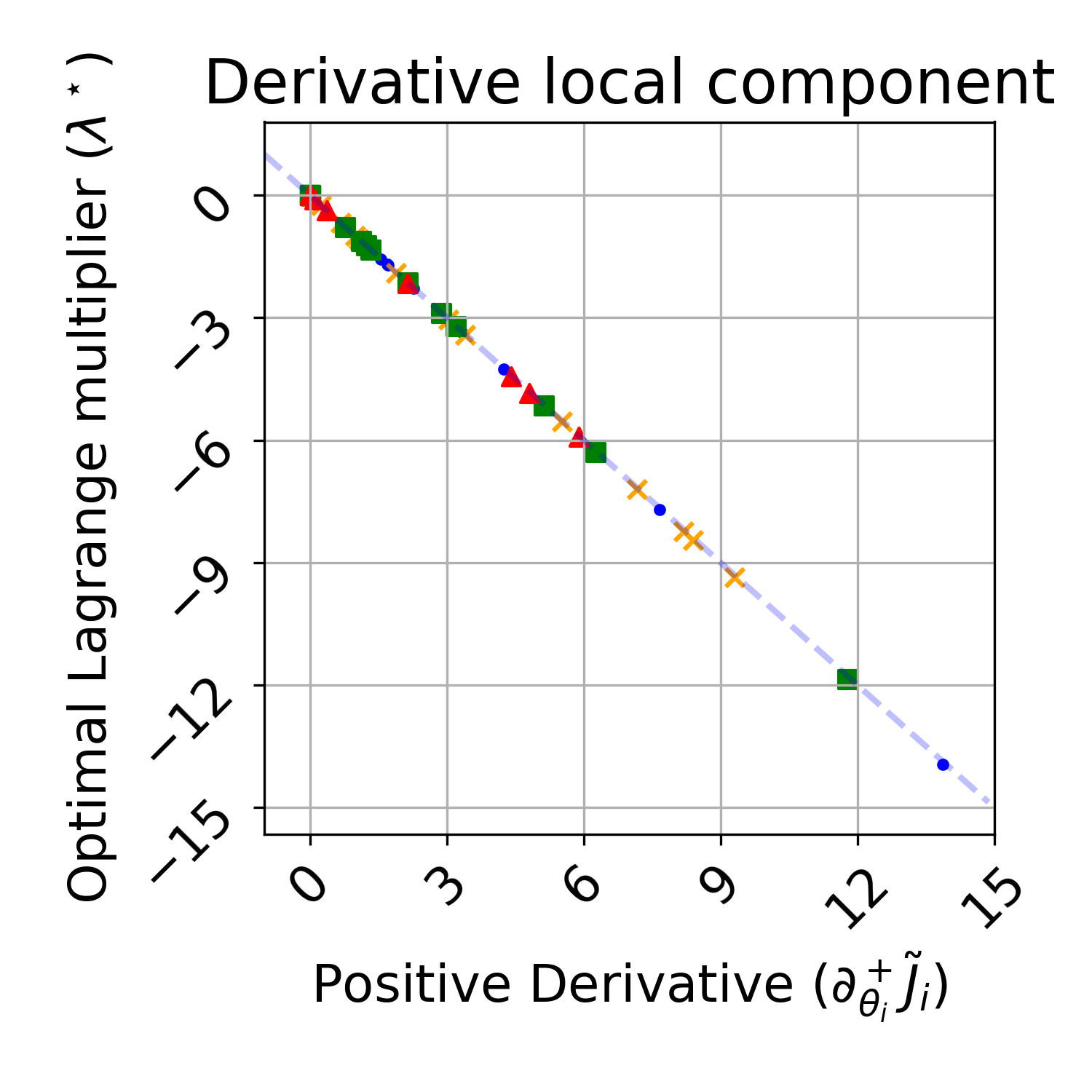}
        \caption{Numerical verification that the finite-difference derivative of the objective with respect to the local noise coincides with the negative of the optimal Lagrange multiplier, confirming exact correspondence up to solver precision.}
        \label{subfig:local approximate gradient component}
    \end{subfigure}
    \hfill
    \begin{subfigure}[t]{0.45\textwidth}
        \centering
        \includegraphics[width=\linewidth]{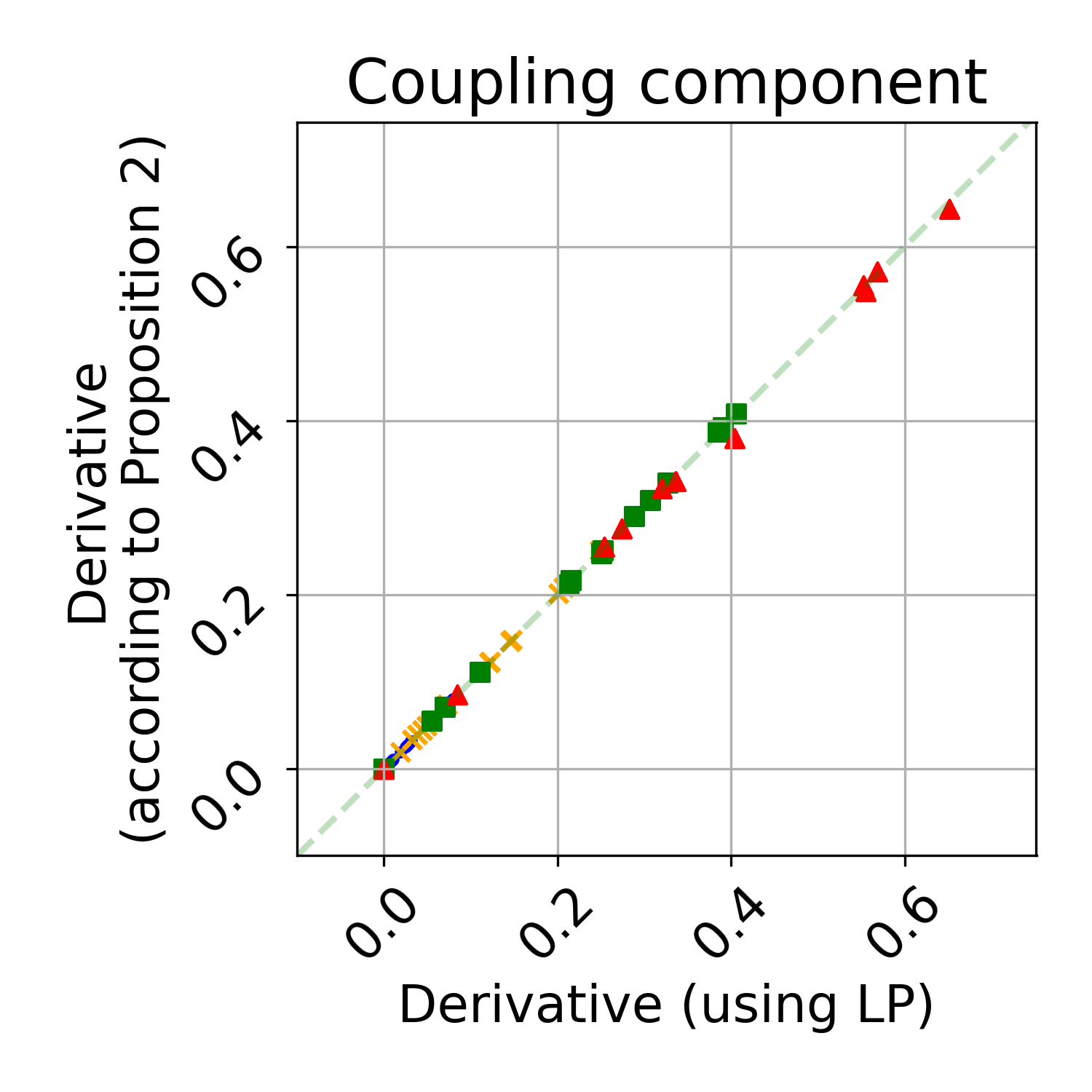}
        \caption{Numerical verification that the finite-difference derivative of the objective with respect to the local noise coincides with the value given by \cref{prop:sign derivatives objective function}, confirming exact correspondence up to solver precision.}
        \label{subfig:lagrange multiplier vs positive derivative}
    \end{subfigure}
    \caption{In these figures we evaluate an exact gradient of $\tilde{J}_i^\ell(\theta)$ using the finite difference method with a very small noise $\epsilon \in \{ -0.00001, 0.00001 \}$. }
    \label{fig:verification gradient exact values}
\end{figure}

\medskip

Additional experimental results are provided in the appendix, including the complete data underlying \cref{fig:evolution reward scalability comparison}, the evolution of offloading behavior in larger systems, comparisons with alternative baselines, and numerical evaluations of the theoretical results from \cref{lemma:bound approximation reward} using linear programming. Together, these supplementary experiments reinforce the conclusions drawn from the main results.

\section{Conclusions}
This paper introduced the DCC framework, a constraint-based approach to decentralized coordination in multi-agent reinforcement learning, and illustrated its potential in the context of task offloading at the wireless edge. Our focus here has been on laying the theoretical foundations and providing preliminary, proof-of-concept experiments to validate the core ideas. While the empirical evaluation is intentionally limited to toy models, the results suggest that constraint-driven implicit coordination can scale better than centralized methods and consistently outperform independent learners.

% We presented a decentralized reinforcement learning framework for coordinating agents in shared-resource environments, applied to task offloading in wireless edge systems. By modeling each agent as a constrained MDP and coordinating through a shared constraint vector, our method enables scalable, communication-efficient learning with theoretical guarantees.

% Experiments show that our approach performs competitively with centralized methods in small systems and scales more effectively in larger ones, consistently outperforming independent learners. Beyond offloading, the framework applies to problems such as dynamic spectrum access, power control, and distributed scheduling, where decentralized agents must share limited resources without explicit communication.

Future work includes extending the framework to support asynchronous updates, exploring richer forms of shared constraints, and conducting broader experiments to further validate the effectiveness of our approach. In this early work we abstract away the system-level implementation of constraint updates, but envision that in practice they could be coordinated through lightweight periodic broadcasts from edge servers or distributed consensus protocols among devices. Exploring these mechanisms remains a valuable direction for future work.

\newpage
\bibliographystyle{plain}
\bibliography{biblio}
\appendix
\newpage
\section{Theoretical proofs}
\subsection{Proof of \Cref{lemma:bound approximation reward}}
\label{subsec:proof bound approximation lemma}
\begin{comment}
First we repeat a result in \cite{simic2009new}:
\begin{theorem}
\label{thm:bound Jensen gap}
    Let $\tilde{p} = \{ p_i \}$ such that $\sum_i p_i = 1$, be a sequence of weights and $\tilde{x} = \{ x_i \}$ a finite sequence of real numbers in a closed interval $I = [a, b]$. Then, for any function $f$ which is convex over $I$ we have 
    \begin{equation}
    \label{eq:bound Jensen gap}
        \sum_i p_i f(x_i) - f\rp{ \sum_i p_i x_i} \leq f(a) + f(b) - 2 f\rp{\frac{a+b}{2}}
    \end{equation}
\end{theorem}

This helps us prove \cref{lemma:bound approximation reward}:
\end{comment}
\begin{proof}
% It is trivial to prove that, given a convex function $d$ on a bounded interval, it exists $\Lambda$ such that $\mid d'' \mid < \Lambda$.

% Moreover, we can observe how, for a twice differentiable function $d$ and a random variable $N$, it can be verified that $$\expvalDist{}{d(N)} - d(\expvalDist{}{N}) \leq \frac{1}{2} \Lambda Var(N)$$

% To conclude the proof, 
\begin{align*}
    J(\pi, \beta) - \hat{J}(\pi, \beta) & = \expvalDist{a_t \sim \pi}{\sum_t \gamma^t \sum_i r_i(s_t, a_t) \mid s_0 \sim \beta} - \expvalDist{a_t \sim \pi}{\sum_t \gamma^t \sum_i \hat{r}_i(s_t, a_t) \mid s_0 \sim \beta}\\
    &= \expvalDist{a_t \sim \pi}{\sum_t \gamma^t \sum_i \mathbb{I}_{a_{t, i} = a_{crowd}} \rp{d(N(a_t) - d(1 + \theta_{-i})}}\\
    &\leq \frac{1}{1-\gamma} \rp{ \sum_i \theta_i  \expvalDist{a_t \sim \pi}{d(N(a_t) - d(\expvalDist{}{N(t)}) \mid a_i = a_{crowd}} }
\end{align*}
In the case with linear penalty function $d$, we can easily conclude that $R_\pi(\beta) = \hat{R}_\pi(\beta)$.
For the nonlinear case, first we show that 
$$\expvalDist{a_t \sim \pi}{d(1 + \theta_{-i}) \mid a_i = a_{crowd}} = d(1 + \theta_{-i}) $$
When $d$ is convex, we can notice that 
\begin{align}
    \expvalDist{a_t \sim \pi}{d(N(a_t) \mid a_i = a_{crowd}}  \leq& \rp{1 - \frac{ \expvalDist{}{N_t\mid a_i = a_{crowd}}-1}{N_{agents}-1}} d(1) + \notag \\ 
    &\quad +\frac{ \expvalDist{}{N_t\mid a_i = a_{crowd}}-1}{N_{agents}-1} d(N_{agents})
    \notag  \\
    =& \frac{ \expvalDist{}{N_t\mid a_i = a_{crowd}}-1}{N_{agents}-1} d(N_{agents}) \notag\\
    =& \frac{ 1 + \theta_{-i} -1}{N_{agents}-1} d(N_{agents}) \notag\\
    =& \frac{\theta_{-i}}{N_{agents}-1} d(N_{agents})
\end{align}
Note that $d(1) = 0$ due to its definition.
This implies that, for convex penalty function $d$ it is verified that 
$$J(\pi, \beta) - \hat{J}(\pi, \beta) \leq \frac{1}{1-\gamma} \sum_i \theta_i \rp{\frac{\theta_{-i}}{N_{agents}-1} d(N_{agents}) - d(1 + \theta_{-i})}  $$
Finally, for concave penalty function $d$, we can easily prove that 
$$\expvalDist{a_t \sim \pi}{d(N(a_t) \mid a_i} \geq \frac{\theta_{-i}}{N_{agents}-1} d(N_{agents})$$

and therefore conclude that, for every nonlinear penalty function $d$ it it verified that 
\begin{equation}
    \left| J(\pi, \beta) - \hat{J}(\pi, \beta) \right| \leq \frac{1}{1-\gamma} \sum_i \theta_i \rp{\frac{\theta_{-i}}{N_{agents}-1} d(N_{agents}) - d(1 + \theta_{-i})} 
    \label{eq:bound error non linear penalty}
\end{equation}
    
\end{proof}

\subsection{Theoretical background for the proof of \cref{lemma:component objective function differentiable}}
Before proving the main result, we first recall the relevant theoretical background, including the statement of Theorem~\ref{thm:envelope theorem constrained problem}. For clarity, we briefly review the envelope theorem, which applies broadly to optimization problems with parameterized constraints and objective functions defined over choice sets with arbitrary topology. A full treatment can be found in Section 3.5 of~\cite{milgrom2002envelope}.

Consider the following maximization program with $k$ parameterized inequality constraints:
\begin{align}
    &V(t) = \sum_{x \in X: g(x, t) \geq 0} f(x, t), \quad \text{where } g: X \times [0, 1] \to \mathbb{R}^k
    \label{eq:constrained maximization program}
    \\
    &X^\star(t) = \{ x \in X : g(x, t) \geq 0, f(x, t) = V(t) \}. \notag
\end{align}
In the reference framework used by \cite{milgrom2002envelope}, $X$ is a convex set, $f$ and $g$ are such that zero duality gap holds, and the Slater constraint qualification is verified at some $\hat{x} \in X$, i.e., $g_h(\hat{x}, t) > 0$ for all $h=1,\ldots,k$. 

The set of saddle points of the Lagrangian over $(x, y) \in X \times \mathbb{R}^k_+$ at parameter value $t$ takes the form $X^\star (t) \times Y^*(t)$, where $X^*(t)$ is the set of solution to \eqref{eq:constrained maximization program} and $Y^*(t)$ is the set of the solutions of the dual program $$Y^*(t) = \arg\min_{y \in \mathbb{R}^*_k} \rp{ \sup_{x \in X} L(x, y, t) }$$

When the zero duality gap condition holds, the value $V(t)$ of the constrained maximization problem equals the saddle of the Lagrangian with parameter $t$, i.e., $V(t)=\min_{y \in \mathbb{R}^k_+} \max_{x \in X} L(x, y, t)=\max_{x \in X} \min_{y \in \mathbb{R}^k_+}  L(x, y, t)$ 

The following result, which is denoted as Corollary 5 in \cite{milgrom2002envelope}, is crucial to prove the differentiability of \eqref{eq:objective function slow timescale}. 
\begin{theorem}
\label{thm:envelope theorem constrained problem}
    %Suppose that $X$ is a convex compact set in a normed linear space, $f$ and $g$ are continuous and concave in $x$,
    Suppose that $X$ is a convex compact set in a normed linear space, $f$ and $g$ are continuous in $x$, for the Lagrangian zero duality condition holds, $f_t(x, t)$ and $g_t(x, t)$ are continuous in $(x,t)$ , and there exists $\hat{x} \in X$ such that $g(\hat{x}, t)  \gg 0$ for all $t \in [0, 1]$. Then:
    \begin{enumerate}
        \item $V$ is absolutely continuous and for any selection $\rp{x^*(t), y^*(t)} \in X^* (t) \times Y^*(t)$, $$V(t) = V(0) + \int_0^t L_t(x^*(s), y^*(s), s) ds$$
        \item $V$ is directionally differentiable, and its directional derivatives equal:
        \begin{align*}
            V'(t+) &= \max_{x \in X^*(t)} \min_{y \in Y^*(t)} L_t(x, y, t) = \min_{y \in Y^*(t)} \max_{x \in X^*(t)} L_t(x, y, t) \quad \text{for } t<1 \\
            V'(t-) &= \min_{x \in X^*(t)} \max_{y \in Y^*(t)} L_t(x, y, t) = \max_{y \in Y^*(t)} \min_{x \in X^*(t)} L_t(x, y, t) \quad \text{for } t>0
        \end{align*}
    \end{enumerate}
\end{theorem}
The result holds since optimal dual variables are proved bounded, and hence from Theorems 4 and 5 in \cite{milgrom2002envelope} can be applied.

\subsection{Proof of \cref{lemma:component objective function differentiable}}
\label{subsec:proof differentiability objective function}
\begin{proof}
In order to prove the desired result for the function $\hat{J}_i(\theta)$ we need to use \cref{thm:envelope theorem constrained problem} and make a distinction between two separate cases, $t=\theta_i$ and $t = \theta_j$.\\
We consider $$V(t)=\min_{y \in \mathbb{R}^k_+} \max_{x \in X} L(x, y, t)=\max_{x \in X} \min_{y \in \mathbb{R}^k_+}  L(x, y, t)$$
Next we define how all the quantities involved in the two different cases.

\subsubsection{$\boldsymbol{t = \theta_i}$}
\label{subsubsec:theta_i}
In this case, all the values of $\theta_j$ for $j\neq i$ are fixed and simply define the function $f$.\\
Consider the following:
\begin{itemize}
    \item $X = L^\gamma (\beta)$: this is the set of stationary distributions given an initial distribuiton $\beta$ and the discount factor $\gamma$
    \item it is safe to assume the existence of a vector $\theta_{max} \in \mathbb{R}$ such that, $\theta_i < \theta_{max}$. This allows us to normalize the values of the parameter $\theta$. A possible value of $\theta_{max}$ could be the cost obtained when always choosing the crowded action. In general, as long as the cost is a bounded function, this property is easily verified
    \item $f(x, t) = f_{\theta_{-i}}(\rho) =  \sum_{s, a} \rho_i(s, a) \hat{r}_i (s, a ; \theta_{-i})$, with $\hat{r}_i : \S \times \A \times [0, 1]^{N-1} \to \mathbb{R}$ defined by the fixed parameters $\theta_{-i}$
    \item $g(x, t) = g(\rho, \theta_i) =  \sum_{s, a} \rho(s, a) c(s, a) - \frac{\theta_i}{\theta_{max}}$, with with $c: \S \times \A \to \mathbb{R}^N$
    \item the Lagrangian is explicitly written as $$L_i(\rho, \lambda, \theta_i) = f_{\theta_{-i}}(\rho) + \lambda_i g(\rho, \theta_i)$$assuming we are studying the differentiability wrt $\theta_i$
\end{itemize}

Now we show that the hypotheses of \Cref{thm:envelope theorem constrained problem} are verified.

\paragraph{X is a convex compact set in a normed linear space}
This choice of the space of the probability distributions, allows us to prove that $X$ is a convex compact set in a normed linear space, as proved in Corollary 10.1 in \cite{altman2021constrained}.

\paragraph{$\boldsymbol{f}$ and $\boldsymbol{g}$ are continuous and concave in $\boldsymbol{x}$}
The continuity and concavity in $x$ of the functions $f$ and $g$ is an immediate consequence of the definition given above.

\paragraph{$\boldsymbol{f_t(x, t)}$ and $\boldsymbol{g_t(x, t)}$ are continuous in $\boldsymbol{(x, t)}$}
The immediate reward function does not depend on $\theta_i$, therefore $f_t(\cdot) = 0$. On the other hand, $g_t(\cdot) = \frac{1}{\theta_{max}}$

\paragraph{Existence of $\boldsymbol{\hat{x} \in X}$ such that $\boldsymbol{g(\hat{x}, t) \gg 0}$ for all $\boldsymbol{t \in [0, 1]}$}

The existence of a point $\hat{x} \in X$ such that $g(\hat{x}, t) \gg 0$ for all $t \in [0, 1]$ corresponds to the existence of a stationary distribution that strictly satisfies all constraints. This condition is met, for example, by a policy that never selects the crowded action—ensuring all constraints are strictly satisfied—provided that $\theta_i > 0$.

\subsubsection{$\boldsymbol{t = \theta_j}$}
\label{subsubsec:theta_j}
In this case, all the values of $\theta_k$ for $k\neq j$ are fixed; they will be considered fixed parameters in the definition of $f$ and $g$.
Consider the following:
\begin{itemize}
    \item $X = L^\gamma (\beta)$ this is the set of stationary distributions
    \item it is safe to assume the existence of a vector $\theta_{max} \in \mathbb{R}$ such that, $\theta_i < \theta_{max}$. This allows us to normalize the values of the parameter $\theta$. A possible value of $\theta_{max}$ could be the cost obtained when always choosing the crowded action. In general, as long as the cost is a bounded function, this property is easily verified
    \item $f(x, t) = f_{\theta_{-i, j}}(\rho, \theta_j) =  \sum_{s, a} \rho_i(s, a) \hat{r}_i (s, a ; \theta_{-i, j}, \theta_j)$, with $\theta_{-i, j} = \sum_{k \neq i, j} \theta_k$
    \item $g(x, t) = g_{\theta_i}(\rho) =  \sum_{s, a} \rho(s, a) c(s, a) - \frac{\theta_i}{\theta_{max}}$, with with $c: \S \times \A \to \mathbb{R}^N$
    \item the Lagrangian is explicitly written as $$L_i(\rho, \lambda, \theta_j) = f_{\theta_{-i, j}}(\rho, \theta_j) + \lambda_i g_{\theta_i}(\rho)$$assuming we are studying the differentiability wrt $\theta_i$
\end{itemize}

To conclude our proof, it suffices to show that the hypotheses of \Cref{thm:envelope theorem constrained problem} are verified for both choices of the parameter.

\paragraph{X is a convex compact set in a normed linear space}
This choice of the space of the probability distributions, allows us to prove that $X$ is a convex compact set in a normed linear space, as proved in Corollary 10.1 in \cite{altman2021constrained}.

\paragraph{$\boldsymbol{f}$ and $\boldsymbol{g}$ are continuous and concave in $\boldsymbol{x}$}
The continuity and concavity in $x$ of the functions $f$ and $g$ is an immediate consequence of the definition given above.

\paragraph{$\boldsymbol{f_t(x, t)}$ and $\boldsymbol{g_t(x, t)}$ are continuous in $\boldsymbol{(x, t)}$}
The continuity of $f_t(x, t)$ in $(x, t)$ is an immediate consequence of the definition of the reward function $r_i$ and the assumption that it is continuous with regard to $\theta_{-i}$ and therefore it is differential wrt $\theta_j, \ \forall  j\neq i$. On the other hand, $g_t(\cdot) = 0$

\paragraph{Existence of $\boldsymbol{\hat{x} \in X}$ such that $\boldsymbol{g(\hat{x}, t) \gg 0}$ for all $\boldsymbol{t \in [0, 1]}$}

The existence of a point $\hat{x} \in X$ such that $g(\hat{x}, t) \gg 0$ for all $t \in [0, 1]$ corresponds to the existence of a stationary distribution that strictly satisfies all constraints. This condition is met, for example, by a policy that never selects the crowded action—ensuring all constraints are strictly satisfied—provided that $\theta_j > 0$.

\subsubsection{Conclusion of the proof}
Thanks to what is showed in \cref{subsubsec:theta_i,subsubsec:theta_j} we know that the objective function $V(t)$ is absolutely continuous with regard to $t$, for all choices of $t$.\\
Consider now the function $$\hat{J}_i^\rho (\theta) = \expvalDist{s, a \sim \rho^\star_i}{\hat{r}_i(s, a; \theta_{-i)}}$$
where $\rho^\star_i$ depends on $\theta$.\\

When we fix all coordinates $\theta_k$ for $k \neq j$, the mapping
$$ \theta_j \mapsto \hat{J}_i^\rho (\theta)$$
coincides with $V(\theta_j)$, which is known to be absolutely continuous.\\
It then follows that $\hat{J}_i^\rho (\theta)$ is differentiable a.e.\\
Finally, note how we can use Theorem 3.3 in \cite{altman2021constrained} to show how an optimal solution $\rho^\star$ for LP is such that the stationary policy $\pi(\rho^\star)$ is optimal for the original constrained problem. 
This implies that, from the optimal stationary distribution we can also retrieve the optimal policy.

This allows us to conclude that the function $\hat{J}_i(\theta)$ is differentiable wrt every component of $\theta$ and therefore that \cref{thm:convergence algorithm} is verified for the problem studied.

\end{proof}

\subsection{Proof of \cref{prop:sign derivatives objective function}}
\label{subsec: proof sign derivatives}

\begin{proof}
    \Cref{thm:envelope theorem constrained problem} states that $V$ is directionally differentiable, and its directional derivatives equal:
    \begin{align*}
        V'(t+) &= \max_{x \in X^*(t)} \min_{y \in Y^*(t)} L_t(x, y, t) = \min_{y \in Y^*(t)} \max_{x \in X^*(t)} L_t(x, y, t) \quad \text{for } t<1 \\
        V'(t-) &= \min_{x \in X^*(t)} \max_{y \in Y^*(t)} L_t(x, y, t) = \max_{y \in Y^*(t)} \min_{x \in X^*(t)} L_t(x, y, t) \quad \text{for } t>0
    \end{align*}
    % In our context, for agent~$i$, the function $V(t)$ corresponds to $\hat{J}_i(\theta)$, while the function $L$ is defined as % represents the function $\hat{J}_i(\theta, \lambda_i, \pi_i)$ considered also in the main proof:
    % $$L(\rho, \lambda, \theta_i, \theta_{-i}) = \sum_{s, a} \rho_i(s, a) \hat{r}_i(s, a; \theta_{-i}) + \lambda \rp{\sum_{s, a}\rho_i(s, a) c_i(s, a) - \frac{\theta_i}{\theta_{max}}}$$

    When considering the parameter $t = \theta_i$ and $\theta_j$ fixed for $j \neq i$, we know that \begin{align*}
    L(\rho, \lambda, \theta_i) &= f_{\theta_{-i}}(\rho) + \lambda_i g(\rho, \theta_i) \\
    &= \sum_{s, a} \rho_i(s, a) \hat{r}_i (s, a ; \theta_{-i, j}, \theta_j) + \lambda_i \rp{\sum_{s, a} \rho(s, a) c(s, a) - \frac{\theta_i}{\theta_{max}}}
    \end{align*}
    It is easily verified that $\frac{\partial L (\theta_i)}{\partial \theta_i} = - \frac{\lambda}{\theta_{max}}$ and therefore when $\theta_i \in (0, \theta_{max})$ it is verified
    \begin{align*}
        V'(\theta_i+) &= -\max_{\rho \in L^\gamma (\beta)} \min_{\lambda \in \mathbb{R}} \frac{\lambda}{\theta_{max}}\\
        V'(\theta_i+) &= -\min_{\rho \in L^\gamma (\beta)} \max_{\lambda \in \mathbb{R}} \frac{\lambda}{\theta_{max}}
    \end{align*}
    As noted in \cite{milgrom2002envelope}, this is a special case for which it is verified that $$V'(t) = V'(\theta_i+) = V'(\theta_i-) = - \frac{\lambda^\star}{\theta_{max}} \leq 0 \text{ a.e.}$$
    This yields the corresponding result, with $M = \theta_{max}$.\\\\
    
    On the other hand, when $t = \theta_j$, with $j\neq i$, we have
    \begin{align*}
    L(\rho, \lambda, \theta_j) &= f_{\theta_{-i, j}}(\rho, \theta_j) + \lambda_i g_{\theta_i}(\rho) \\
    &= \sum_{s, a} \rho_i(s, a) \hat{r}_i (s, a ; \theta_{-i, j}, \theta_j) + \lambda_i \rp{\sum_{s, a} \rho(s, a) c(s, a) - \frac{\theta_i}{\theta_{max}}}
    \end{align*}
    It then follows that
    \begin{align*}
        \frac{\partial }{\partial \theta_j} L(\rho, \lambda, \theta_j)         &= \frac{\partial}{\partial \theta_j} \sum_{s_i, a_i} \rho_i(s_i, a_i) \hat{r}_i (s_i, a_i) \\
        &= \frac{\partial}{\partial \theta_j} \sum_{s_i, a_i} \rho(s_i, a_i) \left( u_i(s_i, a_i) +  \mathbb{I}_{a = a_{\text{crowd}}}(a_i)  d(1 + \theta_{-i})\right)\\
        &= \frac{\partial}{\partial \theta_j} \sum_{s_i} \rho(s_i, a_{crowd}) d(1 + \theta_{-i}) \\ % \geq 0 \quad \forall \rho, \lambda\\
        &= \sum_{s_i} \rho(s_i, a_{crowd}) \frac{\partial}{\partial \theta_j} d(1 + \theta_{-i}) \geq 0 \quad \forall \rho, \lambda
    \end{align*}
    where the final inequality follows from the assumption that the function $d$ is strictly increasing, i.e. $d' > 0$. Moreover, not how the derivative is equal to 0 if and only if the action $a_{crowd}$ is never chosen.\\
    Therefore, we can further conclude that 
    \begin{align*}
    \frac{\partial}{\partial \theta_j} \hat{V} (\theta_j)
    =& \max_{\rho} \min_{\lambda} \sum_{s_i} \rho(s_i, a_{crowd}) \frac{\partial}{\partial \theta_j} d(1 + \theta_{-i})\\
    =& \sum_{s_i} \rho^\star(s_i, a_{crowd}) \frac{\partial}{\partial \theta_j} d(1 + \theta_{-i}) \\
    =& \mathbb{P}\rp{ a_{i} = a_{crowd} \mid a_{i, t} \sim \pi_i^\star} \cdot \frac{\partial}{\partial \theta_j} d(1 + \theta_{-i})\\
    =& \theta_i  \frac{\partial}{\partial \theta_j} d(1 + \theta_{-i})
    \end{align*}
    where the optimal stationary distribution $\rho^\star$ is a function of the constraint vector $\theta$.    
\end{proof}

\newpage

\section{Toy model considered}
\label{sec:appendix toy model description}
\paragraph{Description of the toy model}
%%%%%%%%%%%%%%%%%%%%%%%%%%%%%%%%%%%%%%%%%%%%%%
%%%%%%%%%%%%%%%%%%%%%%%%%%%%%%%%%%%%%%%%%%%%%%%%%%%%%%%%%%%%%%%%%%%%%%%%%%%%%%%%%%%%%%%%%%
% EDIT: added Markov chain formulation to describe the energy units required to process data
\begin{figure}[t]
    \centering
    \includegraphics[width = .65 \columnwidth]{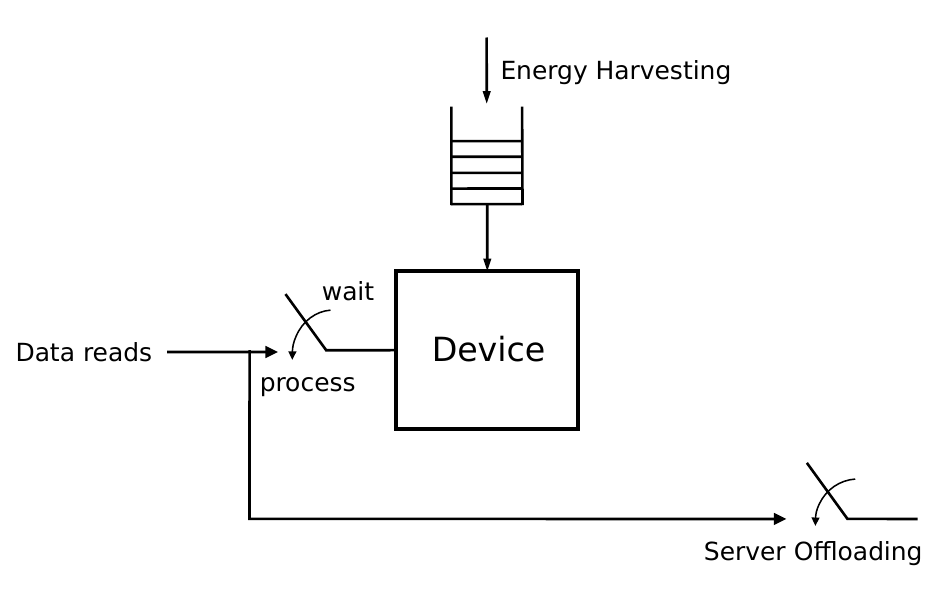}
    \caption{\small{Device performing local processing of data batches with energy harvesting and offloading.}}
    \label{fig:system}
\end{figure}

The device-server scheme is represented in Fig.~\ref{fig:system} for a single device: the device can read data batches, process them locally or offload them to the edge server. Thus, data batches are either read and processed locally on the device or read and offloaded. Time steps are discrete with index $t=1,2,\ldots$. To process data on the device at time $t$, a certain amount of energy units, denoted as $C_t$, is required. $C_t$ is modeled as a Markov chain. When the device offloads to the edge server, the data processing task is offloaded and it has zero energy cost for the device.
However, sending and retrieving processed data from the server may take some time, which depends on how many devices have offloaded the data. This will result in a reward that accounts for the delay in receiving the processed data. The energy to transmit data to the server is assumed negligible compared to that spent for local data processing.
The device is equipped with a battery of capacity $B > 0$ energy units. Additionally it can harvest a certain number of energy units $H_t$ per timestep $t$. As for the processing cost, the energy fetched per timestep, namely the {\em harvesting rate}, is described by a Markov chain.

\paragraph{Description of the CMDP}
For each device we can introduce a Constrained Markov Decision Process (MDP) to model the system. The state of the device $i$ at time $t$ is denoted as $s_{i, t} =(x_{i, t}, ,e_{i, t})$, where $x_{i, t}$ represents the age of information (AoI), $e_{i, t}$ is the device battery level. 

Note that $x_{i, t}$  is the AoI of the last data batch processed by a tagged device and the AoI is measured at the end of each time slot. It holds $x_{i, t}=1$ when the device has just processed fresh data. Afterwards, the AoI increases of one timestep at every time slot $t$ until either the device fetches a new data batch or the offloading occurs. In both cases, at the end of the data processing the AoI is reset to $1$. 

The freshness function, denoted as $u_i(x)$, represents the utility of processing a data batch $x$ time units after the processing of the last batch. The function $u_i(\cdot)$ is bounded and non-increasing which means that beyond a certain value of AoI, any further processing delay would have minimal impact on the utility derived from the data. Consequently, it is assumed that there exists a positive constant $M$ such that $u_i(M+k) = u_i(M)$ for any positive value of $k$, as it allows us to focus on the relevant time window $\{ 1, \dots, M \}$ where the utility function exhibits meaningful changes in value in relation to the freshness of the data. 

Finally, it is possible that the battery available at timeslot $t$ is not sufficient to terminate the computation immediately, namely, $e_t + H_t - C_t < 0 $. In this case, a delay is incurred in order to harvest a sufficient amount of energy and complete the processing of the current batch. The corresponding timeslot has a random duration, due to the stochastic nature of energy harvesting.

The state space of agent $i$ is denoted $\S_i=\{1,\ldots,M\}\times \{0,\ldots,B\}$. 
\begin{figure}[t]
    \centering
    \includegraphics[scale = 0.4]{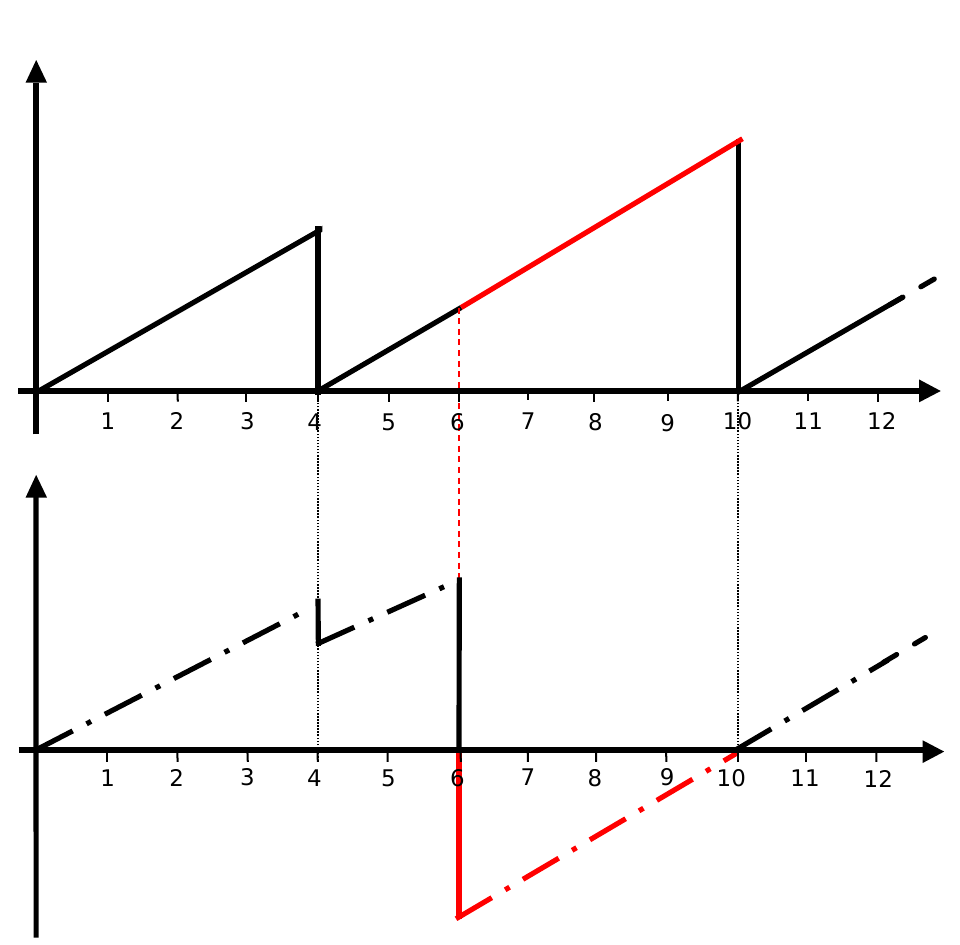}
    \put(-190,160){$x_t$}
    \put(-190,85){$e_t$}
    % \put(-460,203){$x_t$}\put(-235,97){$e_t$}
    % \put(-10,25){$t$}\put(-10,110){$t$}
    \put(-125,72){\begin{rotate}{90}{\small $a_4\!=\!1$}\end{rotate}}
    \put(-98,72){\begin{rotate}{90}{\small $a_6\!=\!1$}\end{rotate}}
    \put(-100,5){\begin{rotate}{90}{\small $\phi_6 > 0$}\end{rotate}}
    % \put(-92,68){\begin{rotate}{90}{\small $a_{12}\!=\!1$}\end{rotate}}\put(-86,103){\small $\phi_{12}$}
    %  \put(-32,68){\begin{rotate}{90}{\small $a_{18}\!=\!1$}\end{rotate}}\put(-40,103){\small $\phi_{18}$}
     % \put(-2,70){\begin{rotate}{90}{\small $a_{21}\!=\!1$}\end{rotate}}
    \caption{A sample path of the process $s_t=(x_t,e_t)$ for $H \equiv 1$ for a single agent: in red the vacation periods where the battery is empty. After the battery level gets negative at time $t=6$ a total of $4$ timesteps of recharging are needed to complete the execution of the task.}
    \label{fig:renewal}
\end{figure}

The action set is $\A_i=\{\text{"read", "local processing", "offload" }\}$. The action taken by agent $i$ at time $t$ is denoted $a_{i, t}$. If $a_{i, t}=1$,  device $i$ fetches a new data batch, which is processed at energy cost $C_t>0$. Finally, the dynamics of the AoI for data batches is 
\begin{equation}
x_{t+1}= 
\begin{cases} 
\;1 & \quad \mbox{if} \; a_t \in \{ \text{"local processing" and } e_{i, t} + H_{i, t} - C_{i, t} \geq 0 , \text{"offloading"} \}  \\
\;\min \{ x_t+1, M \} & \quad \mbox{if} \; a_t \in \{ \text{"read"}, \text{"local processing" and } e_{i, t} + H_{i, t} - C_{i, t} < 0 \} \\
\end{cases}\nonumber
\end{equation}
The AoI at the renewal instants is also called {\em peak AoI} (\cite{barakat2019measure}) in the literature.

As described in \cref{sec:system model} the reward function is composed of two components: a local utility function and a component which represents the congestion.
For the local components, we consider the reward function which penalizes having negative battery level:
\begin{equation}
    u_i(s_i, a_i) = 
    \begin{cases}
        x_i - e_i \quad& e_i <0\\
        x_i \quad &e_i \geq 0
    \end{cases}
    \label{eq:local utility function}
\end{equation}
For the sake of simplicity, for this work we consider a simple function also for the penalty component which depends on the congestion. In particular, 
\begin{equation}
    d(n) = n ^ \alpha
\end{equation}
with $\alpha \in \mathbb{R}$. Clearly, for $\alpha = 1$ we are considering the case with linear penalty. 

Finally, the cost function is defined as 
\begin{equation}
    c_i(s_i, a_i) = \begin{cases}
        1 \quad a_i = \text{"offloading"}\\
        0 \quad \text{otherwise}
    \end{cases}
\end{equation}

\newpage
\section{Additional numerical experiments}
We introduce some additional experiments. Whenever reported, the reward is always normalized so that 1 corresponds to the discounted reward obtained by \icAlg after $10^5$ learning steps.

\subsection{Evolution scalability comparison: additional data}
We plot the reward evolution corresponding to the boxplot in \cref{subfig:scalability comparison}. This highlights why MAPPO results were excluded from the main figure to avoid distortion. The plot also shows that IQL converges quickly due to the small system size, while \icAlg continues to improve over time as the constraint approaches its optimal value.
\begin{figure}[h]
    \centering
    \includegraphics[width=0.99\linewidth]{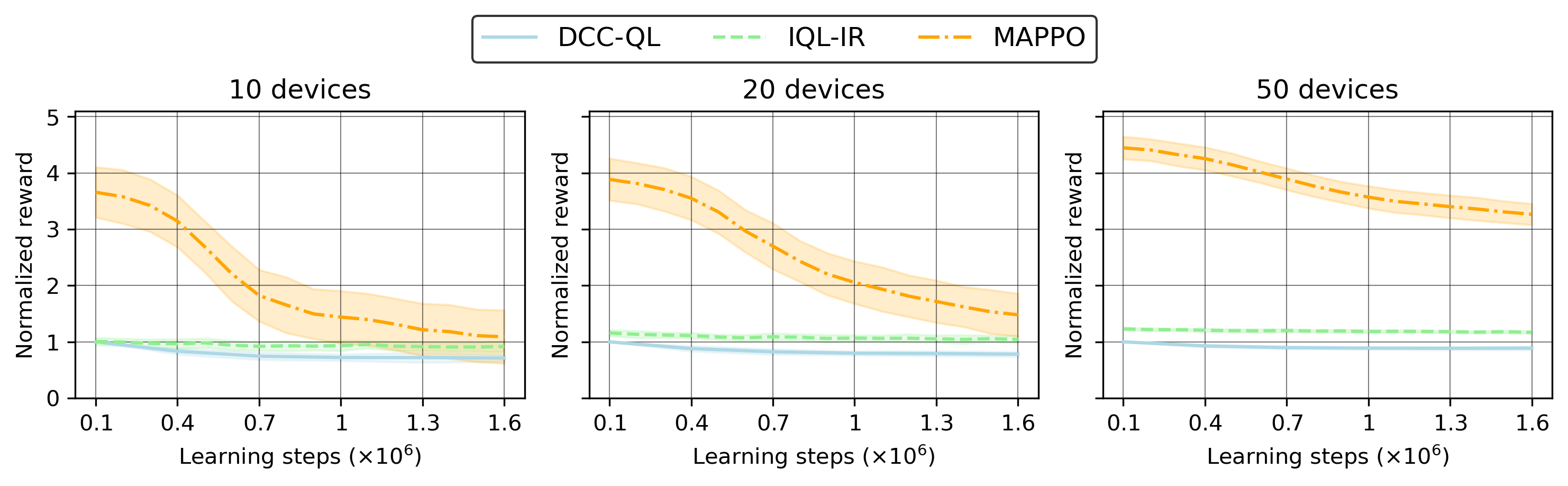}
    \caption{Evolution of the normalized reward over training episodes for different methods (corresponding to the boxplot in \cref{subfig:scalability comparison}).}
    \label{fig:evolution reward scalability comparison}
\end{figure}

\subsection{Evolution offloading action: additional data}
For completeness, we also plot the corresponding results in \cref{subfig:offloading action comparison} for $N=20$ and $N=50$.
\begin{figure}[h]
    \centering
    \includegraphics[width=0.99\linewidth]{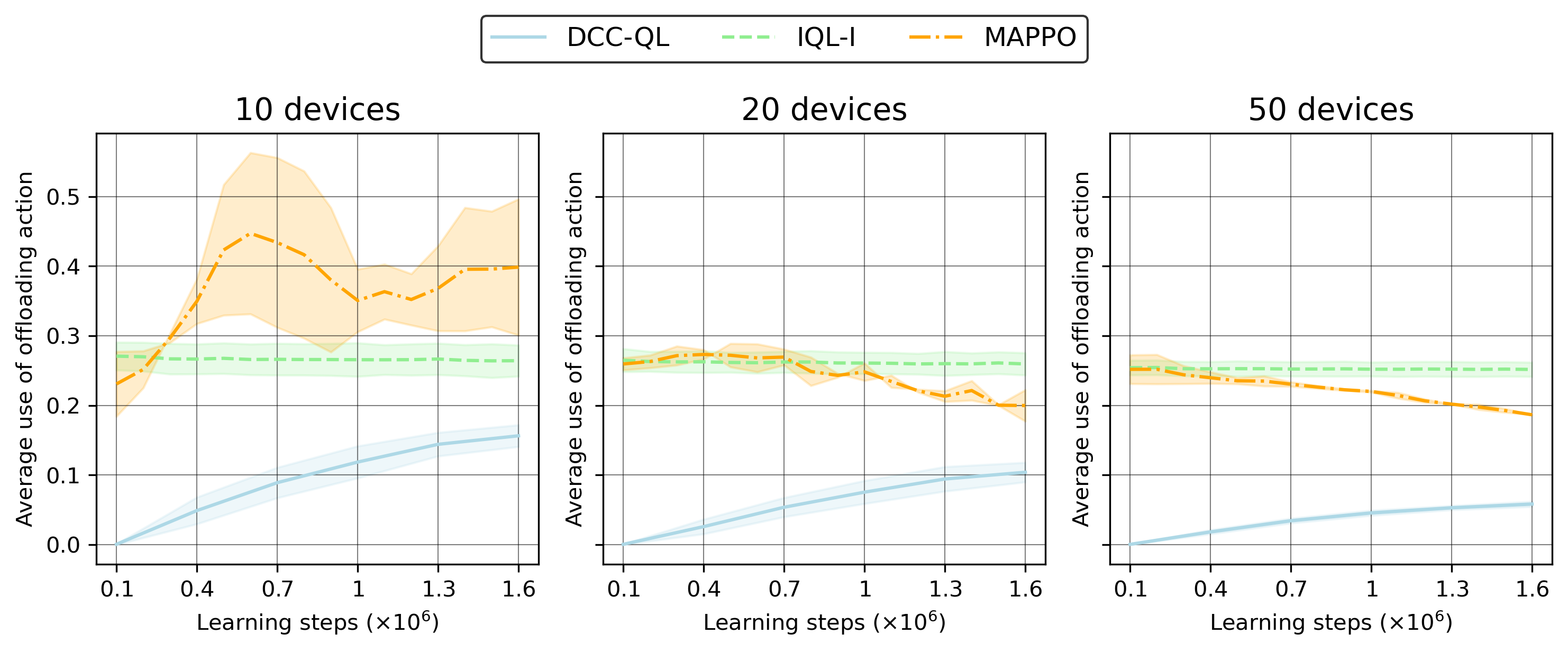}
    \caption{Evolution of the frequency of the offloading action in systems with $10, 20$ and $50$ devices.}
    \label{fig:comparison offloading action complete}
\end{figure}
\subsection{Comparison with MAPPO over longer timeframes}
We include this comparison in \Cref{fig:comparison scalability longer MAPPO} to illustrate how MAPPO improves with extended training. While MAPPO achieves lower rewards after $10^7$ learning steps, this comes at the cost of substantially more samples. In contrast, \icAlg makes better use of available data thanks to its Q-learning foundation, achieving strong performance with far fewer interactions. Even with extended training, \icAlg remains consistently better across system sizes, highlighting its sample efficiency and robustness.

\begin{figure}[h]
    \centering
    \includegraphics[width=0.9\linewidth]{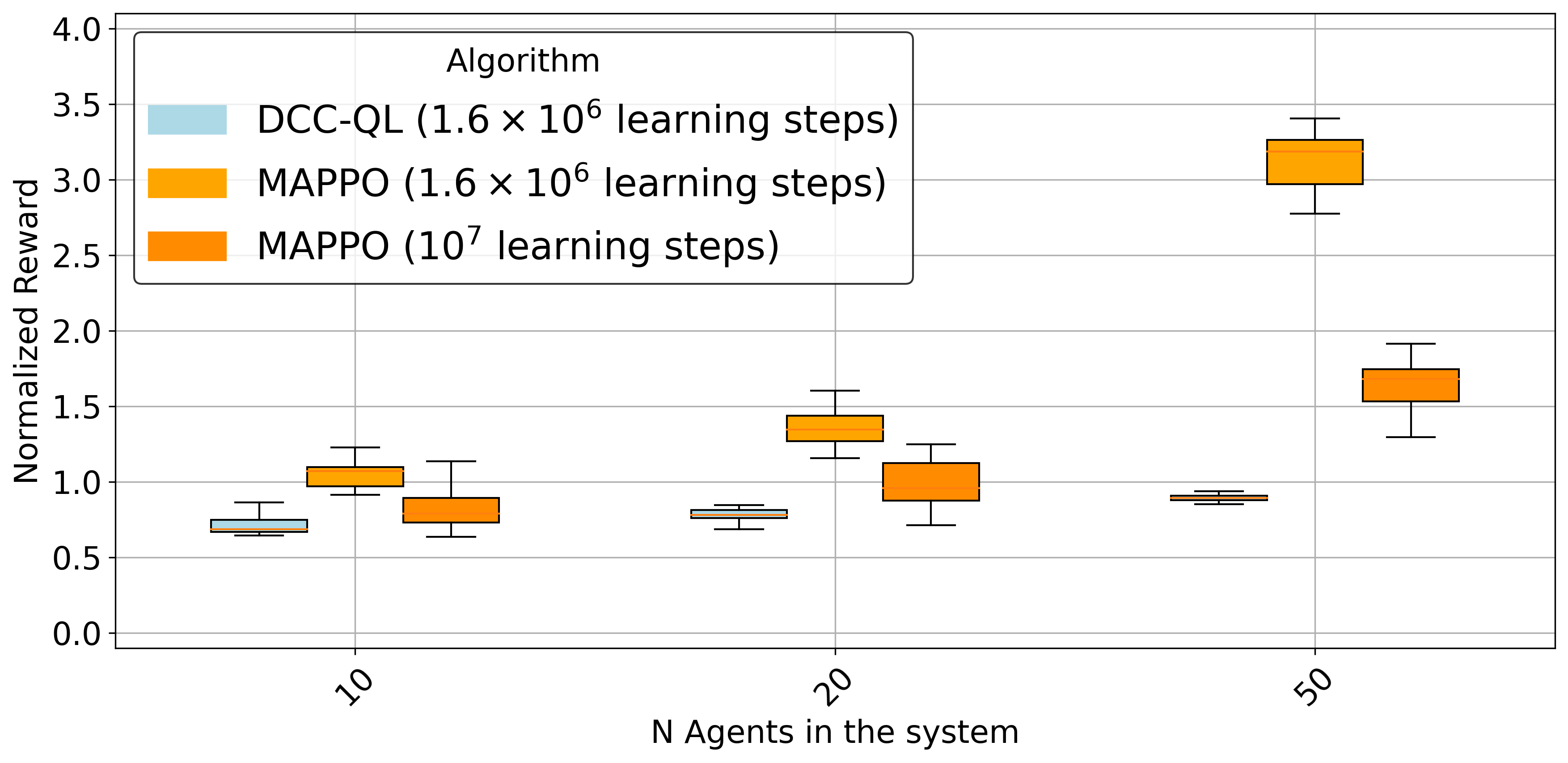}
    \caption{Comparison of \icAlg with MAPPO after $1.6 \times 10^6$ learning steps (as in \cref{subfig:scalability comparison}) and after $10^7$ learning steps.}
    \label{fig:comparison scalability longer MAPPO}
\end{figure}

\subsection{Additional algorithms for \cref{subfig:scalability comparison}}
For completeness, \cref{fig:full comparison scalability} also reports results for IQL with a common reward and for a centralized baseline using A2C. Sharing a global reward signal can, in principle, improve coordination among agents by aligning their incentives and reducing the risk of selfish behavior. However, this comes at the cost of increased variance in the learning signal, and in practice the benefits are limited in our setting. The centralized A2C results are obtained for small system sizes using the Stable Baselines implementation \cite{stable-baselines3} with standard hyperparameters. This corresponds to a single agent observing the entire system and making joint decisions for all devices. To simplify the setup, we did not explicitly restrict illegal actions (e.g., processing with negative energy levels), but instead imposed large penalties, which is generally a suboptimal approach. While this method performs reasonably well with 10 devices, it fails to scale: for 20 devices, the average rewards were already about 15 times higher than the normalization factor, making meaningful comparison impossible. Thus, centralized methods serve here primarily as conceptual baselines rather than practical solutions.
\begin{figure}[h]
    \centering
    \includegraphics[width=\linewidth]{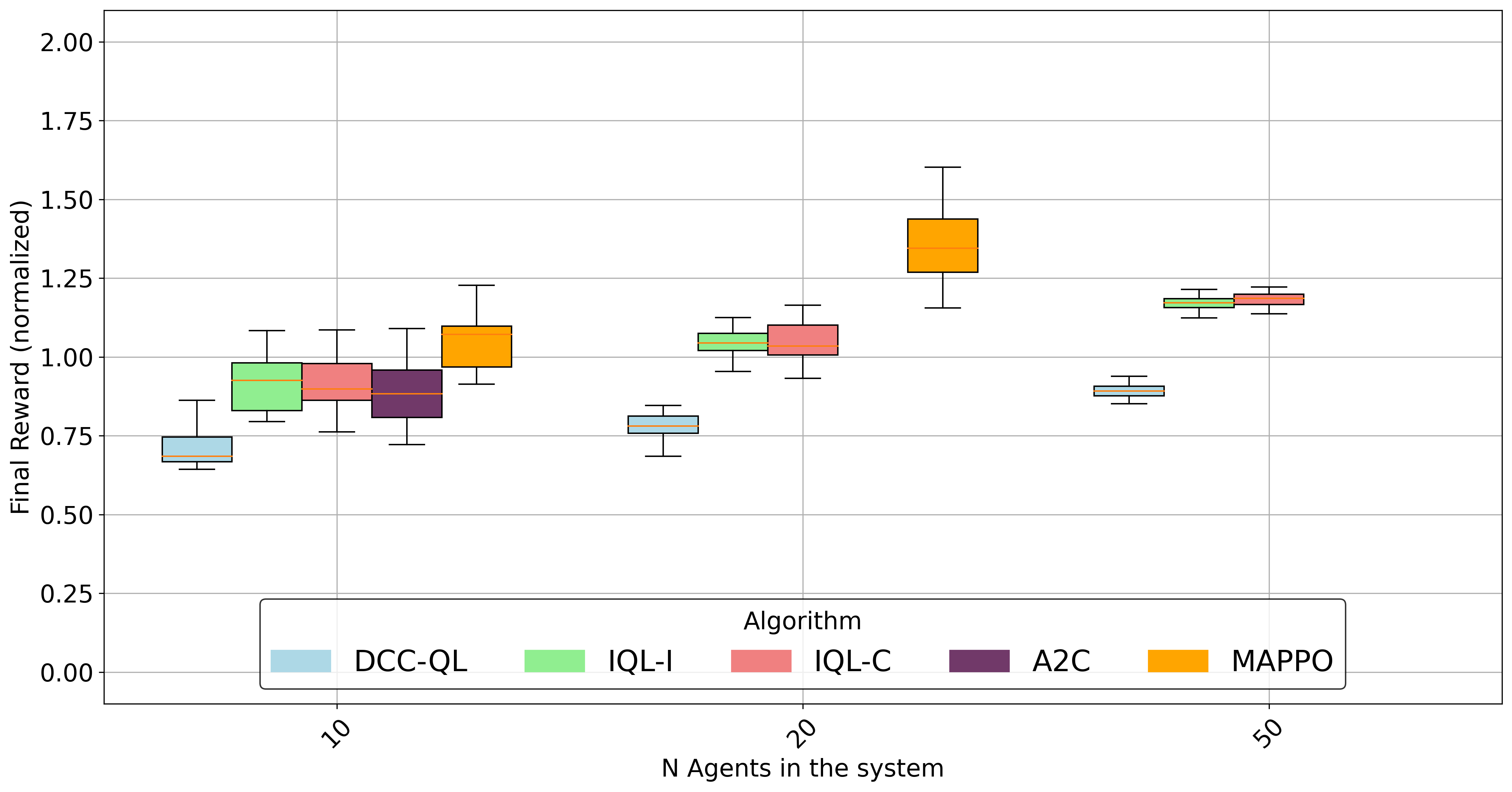}
    \caption{Performance comparison as in \cref{subfig:scalability comparison}, including IQL with common reward and centralized A2C ( only for small system sizes).}
    \label{fig:full comparison scalability}
\end{figure}

\subsection{Verification of the Approximate Reward Bound}
In this experiment, we considered small environments with $10$ devices, where the linear program (LP) can be applied reliably to compute the optimal policy. For each device, we compared the discounted reward obtained using the approximated reward definition in \eqref{eq:approximated reward} with the corresponding component of the true global reward defined in \eqref{eq:global reward}. This comparison allows us to quantify the error introduced by the approximation and to assess its alignment with the theoretical analysis.

To this end, we measured the relative error between the two evaluations and compared it with the theoretical bound established in \cref{lemma:bound approximation reward}. The results confirm that the observed error is consistently within the predicted range, up to small numerical deviations due to stochastic evaluation. Importantly, the special case $\alpha = 1$, where the penalty function is linear and the theoretical bound reduces to zero, was verified exactly in our experiments. For all nonlinear values of $\alpha$, the empirical errors remained strictly below the bound, thereby validating the theoretical result in practice.

These findings provide additional evidence for the robustness of the approximation framework, showing that the theoretical guarantees extend to practical instances solved via the LP in small environments.
\begin{figure}[h]
    \centering
    \begin{subfigure}[t]{0.45\textwidth}
        \includegraphics[width=\linewidth]{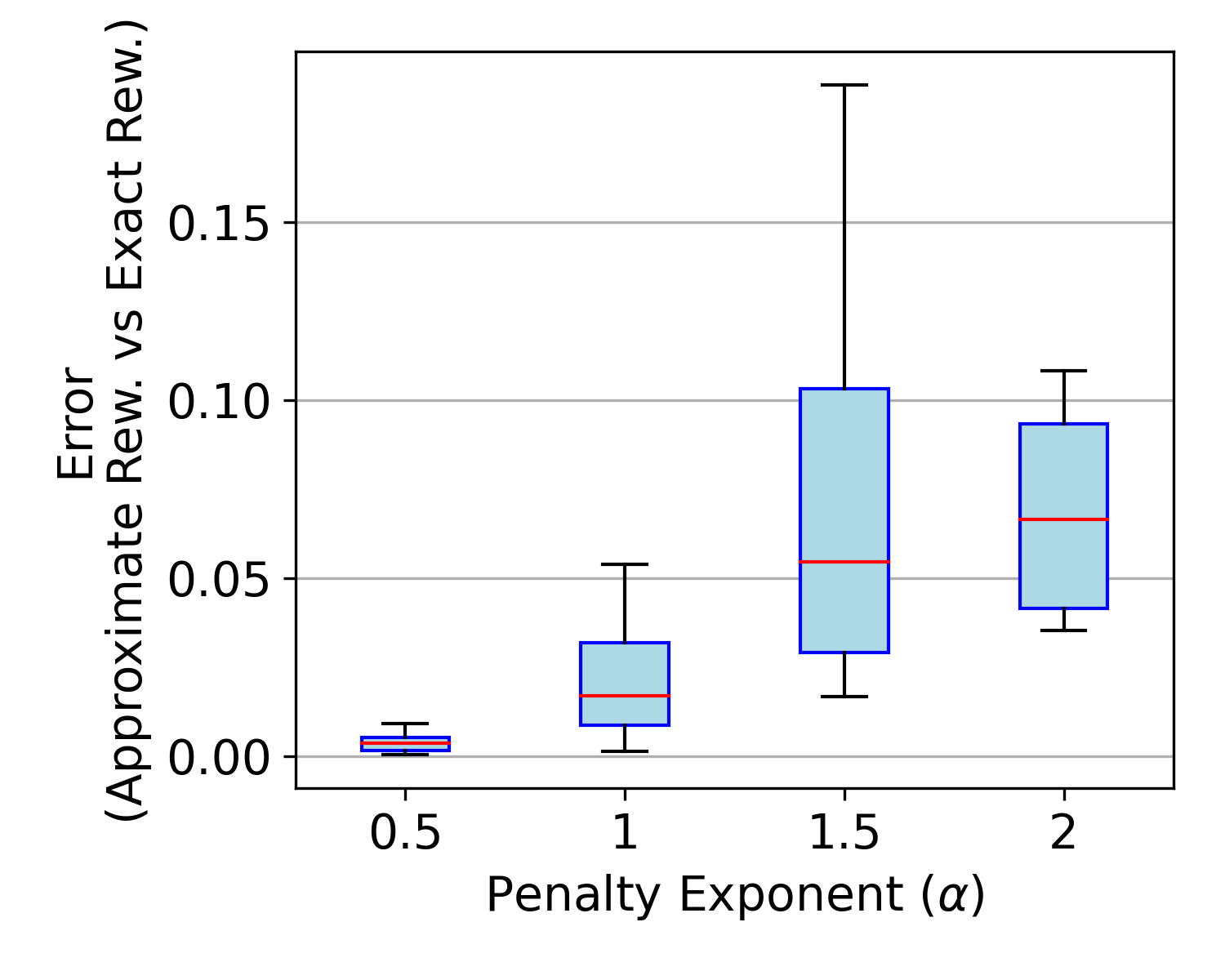}
        \caption{Boxplot of the relative error as a function of the penalty exponent $\alpha$}
        \label{subfig:comparison error boxplot}
    \end{subfigure}
    \hfill
    \begin{subfigure}[t]{0.45\textwidth}
        \includegraphics[width=\linewidth]{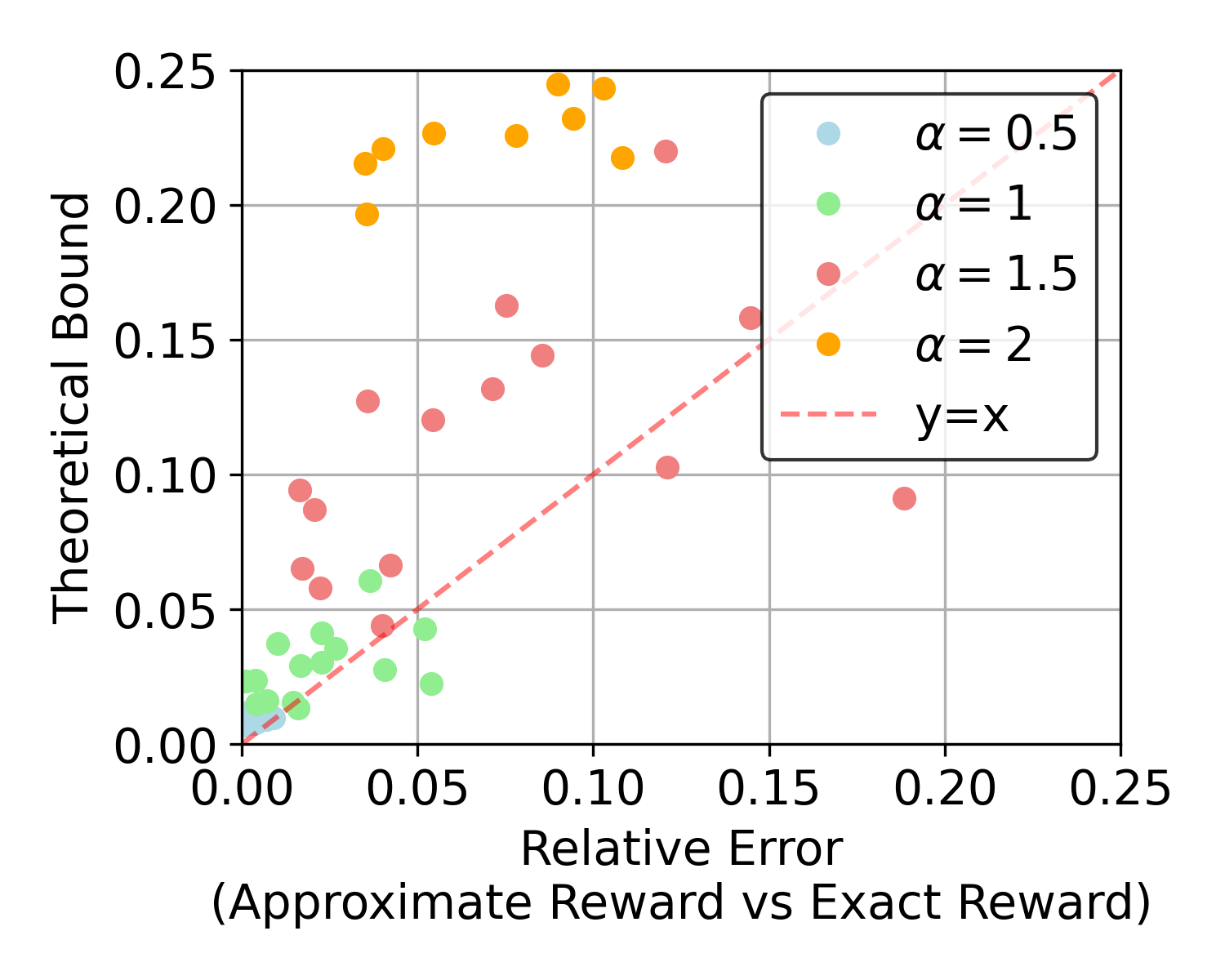}
        \caption{Relative error compared with the theoretical bound from \cref{lemma:bound approximation reward}, showing that the error remains within the predicted range in most cases.}
        \label{subfig:comparison error theoretical bound}
    \end{subfigure}
    \caption{Evaluation of the approximation error between the discounted reward computed with the exact reward and with the approximated reward.}
    \label{fig:comparison error LP}
\end{figure}

\newpage
\section{Extended Algorithm}
\label{secApp:extended algorithm}
The following pseudocode corresponds to the version of the algorithm that uses finite differences to compute the gradient of the objectove function. This is the version used in the nu,erical experiments.
\begin{algorithm}[ht]
\caption{Three-Timescale Decentralized Constrained MARL}
\label{alg:three_timescale_explicit_ext}
\begin{algorithmic}[1]
\State \textbf{Initialize:} Initial constraints $\theta_i^0$, multipliers $\lambda_i^0$, policies $\pi_i^0$
\Function{evaluate\_optimal\_constrained\_policy}{$i, \theta_i, \theta_{-i}, n$}
    \For{$m = 0, 1, \dots, M(n)$} \Comment{Intermediate timescale — $\lambda$ update}
        \For{$t = 0, 1, \dots, T(m,n)$} \Comment{Fastest timescale — policy update}
            \State Collect trajectories using policy $\pi_i^{t,m,n}$
            \State Compute total reward: $r_i^{\text{total}} (s_i, a_i; \theta_{-i}) = \hat{r}_i (s_i, a_i; \theta_{-i}) + \lambda_i c_i (s_i, a_i)$\\
            \State Update policy: $\pi_i^{t+1,m,n} \gets$ RL step on $r_i^{\text{total}}$
        \EndFor
        \State $\hat{K}_i \gets$ $K_i(\pi_i^{\cdot,m,n})$  \Comment{Estimate long term cost of current policy}  
        \State $\lambda_i^{m+1,n} \gets \lambda_i^{m,n} + \alpha_k (\hat{K}_i - \theta_i)$
    \EndFor
    \State \Return $  \hat{J}_i(\theta_i, \theta_{-i}) \gets$ evaluate local policy $\pi_i^{T(M(n), n), M(n)}$
\EndFunction
\\\\
\For{$n = 0, 1, 2, \dots$} \Comment{Slowest timescale — constraint update}
    \For{each agent $i = 1, \dots, N$}
        \State $ \hat{J}_i (\theta_i^n, \theta^n_{-i})  \gets$ \Call{evaluate\_optimal\_constrained\_policy}{$i,\theta_i^n, \theta_{-i}^n ,n$}
    \EndFor
    \State $\epsilon \gets$ RandomNoise(0, $\sigma$)
    \For{each agent $i = 1, \dots, N$}
    \State $  \hat{J}_i (\theta_i^n + \epsilon_i, \theta^n_{-i}) \gets$\Call{evaluate\_optimal\_constrained\_policy}{$i, \theta_i^n + \epsilon_i, \theta_{-i}, n$}
    \State $ \hat{J} (\theta_i^n, \theta^n_{-i} + \epsilon_i) \gets$\Call{evaluate\_optimal\_constrained\_policy}{$i, \theta_i^n, \theta_{-i} + \epsilon_i, n$}
    \EndFor
    \For{each agent $i$}
        \State $ \frac{\partial}{\partial \theta_i}   \hat{J}_i(\theta_i^n, \theta_{-i}^n) \gets \rp{ \hat{J}_i(\theta_i^n + \epsilon_i, \theta_{-i}^n) -  \hat{J}_i(\theta_i^n, \theta_{-i}^n)}/\epsilon_i$
        \State $ \frac{\partial}{\partial \theta_{-i}}  \hat{J}_i (\theta_i^n, \theta_{-i}^n) \gets \rp{ \hat{J}_i(\theta_i^n, \theta_{-i}^n + \epsilon_i) -  \hat{J}_i(\theta_i^n, \theta_{-i}^n)}/\epsilon_i$
        \State $\nabla \frac{\partial  }{\partial \theta_i} \hat{J}(\theta) = \frac{\partial \hat{J}_i (\theta)}{\partial \theta_i}  + \sum_{j \neq i} \frac{\partial \ \hat{J}_j(\theta)}{\partial \theta_{-j}}$
        \State $\theta_i^{n+1} \gets \theta_i^n + \eta_n \nabla_{\theta_i}  \hat{J}_i (\theta_i^n)$
    \EndFor
\EndFor
\end{algorithmic}
\end{algorithm}

\newpage
\section{Extended Algorithm with exact evaluations of $\boldsymbol{\lambda_i^\star}$}
The following pseudocode corresponds to the version of the algorithm that assumes access to the exact values of $\lambda_i^\star$, obtained through the learning process, for computing the gradient.
\begin{algorithm}[ht]
\caption{DCC Framework}
\label{alg:three_timescale_explicit_ext}
\begin{algorithmic}[1]
\State \textbf{Initialize:} Initial constraints $\theta_i^0$, multipliers $\lambda_i^0$, policies $\pi_i^0$
\Function{OptimizeLocalCMDP}{$i, \theta, n$}
    \For{$m = 0, 1, \dots, M(n)$} \Comment{Intermediate timescale — $\lambda$ update}
        \For{$t = 0, 1, \dots, T(m,n)$} \Comment{Fastest timescale — policy update}
            \State Collect trajectories using policy $\pi_i^{t,m,n}$
            \State Compute total reward: $r_i^{\text{total}} (s_i, a_i; \theta_{-i}) = \hat{r}_i (s_i, a_i; \theta_{-i}) + \lambda_i c_i (s_i, a_i)$\\
            \State Update policy: $\pi_i^{t+1,m,n} \gets$ RL step on $r_i^{\text{total}}$
        \EndFor
        \State $\hat{K}_i \gets$ $K_i(\pi_i^{\cdot,m,n})$  \Comment{Estimate long term cost of current policy}  
        \State $\lambda_i^{m+1,n} \gets \lambda_i^{m,n} + \alpha_k (\hat{K}_i - \theta_i)$
    \EndFor
    \State $\hat{J}_i(\theta) \gets$ evaluate long term reward of policy $\pi_i^{T(M(n), n), M(n)}$
    \State \Return $\hat{J}_i (\theta), \lambda_i^{M(n), n}, \pi_i^{T(M(n), n), M(n), n}$
\EndFunction
\\\\
\State Compute $\theta_i^{MAX}$ for each agent $i$: $\theta_i^{MAX} \gets K_i(\pi_i^{\text{always offload}})$
\For{$n = 0, 1, 2, \dots$} \Comment{Slowest timescale — constraint update}
    \For{each agent $i = 1, \dots, N$}
    \State $  \hat{J}_i (\theta^n), \lambda_i^\star, \pi_i^\star \gets$\Call{OptimizeLocalCMDP}{$i, \theta^n, n$}
    \EndFor
    \For{each agent $i$}
        \State $ \frac{\partial}{\partial \theta_i}   \hat{J}_i(\theta^n) \gets - \frac{\lambda_i^\star}{\theta_i^{MAX}}$
        \For{each agent $j \neq i$}
        \State $ \frac{\partial}{\partial \theta_{j}}  \hat{J}_i (\theta^n) \gets \theta_i \frac{\partial}{\partial \theta_j} d(1 + \theta_{-i})$
        \EndFor
    \EndFor
    \For{each agent $i= 1, \dots, N$}
        \State $\nabla \frac{\partial  }{\partial \theta_i} \hat{J}(\theta) = \frac{\partial \hat{J}_i (\theta)}{\partial \theta_i}  + \sum_{j \neq i} \frac{\partial \ \hat{J}_j(\theta)}{\partial \theta_i}$
        \State $\theta_i^{n+1} \gets \theta_i^n + \eta_n \nabla_{\theta_i}  \hat{J}_i (\theta_i^n)$
    \EndFor
\EndFor
\end{algorithmic}
\end{algorithm}

\newpage
\section{Numerical experiments hyperparameters}

\paragraph{Parameters of the environment}
To keep a simple system, we considered both harvesting probability and processing cost probability, in an interval respectively $[ min_H, max_H]$ and $[min_C, max_C]$. The following table describes the values among which we sampled the values for each experiment.
\begin{table}[h!]
    \centering
    \begin{tabular}{c|c}
    \textbf{Parameter} & \textbf{Sample set}\\
    \hline
        M & 15  \\
        B &  15 \\
        $min_H$ & \{0, 1\}\\
        $min_H$ & \{1, 2, 3\}\\
        $min_H$ & \{1\}\\
        $max_C$ & \{5, 7, 10\}\\
        $\gamma$ & 0.95
    \end{tabular}
    \label{tab:placeholder}
\end{table}

\subsection{Learning parameters}
Table~\ref{tab:learning_params_QL} reports the learning parameters used in the experiments for IQL and \icAlg. We tested several learning and exploration rates, but the values shown here yielded the best performance. Symbols in parentheses correspond to those used in \cref{alg:three_timescale_explicit_ext}. For \icAlg, the parameters $\alpha_k$, $\sigma$, and $\eta_n$ denote the initial values of sequences that decay according to the convergence conditions of the algorithm; these are not applicable to IQL.

\begin{table}[h!]
    \centering
    \begin{tabular}{c|c|c}
    \textbf{Parameter} & \textbf{\icAlg} & \textbf{IQL}\\
    \hline
        Learning rate Q-learning & 0.5 & 0.05 \\
        Exploration rate Q-learning & 0.05 & 0.05\\ 
        Exploration decaly Q-learning & 0.95 & 0.95 \\
        LM learning rate ($\alpha_k$) & 1 & \\
        Standard deviation constraint noise ($\sigma$) & 0.05 & \\
        Constraint learning rate ($\eta_n$) & 0.25 & 
    \end{tabular}
    \caption{Learning parameters for \icAlg and IQL}
    \label{tab:learning_params_QL}
\end{table}

For MAPPO, we used the implementation in \cite{guo2025socialjax} for the custom environment described in \cref{sec:appendix toy model description}. 
\begin{table}[]
    \centering
    \begin{tabular}{c|c}
        &  MAPPO  \\
\hline
         LR & 0.001  \\
NUM\_ENVS & 8 \\
TOTAL\_TIMESTEPS & 1.6e7 \\
UPDATE\_EPOCHS & 2 \\
NUM\_MINIBATCHES & 64 \\
GAMMA & 0.95 \\
GAE\_LAMBDA & 0.95\\
CLIP\_EPS & 0.2 \\
ENT\_COEF & 0.01 \\
VF\_COEF & 0.5 \\
MAX\_GRAD\_NORM & 0.5 \\
ACTIVATION & relu
    \end{tabular}
    \caption{Learning parameters MAPPO}
    \label{tab:placeholder}
\end{table}{l|c|}

\newpage

\end{document}